\documentclass{article} 
\usepackage{arxiv_macros}

\usepackage[numbers]{natbib}
 \bibpunct[, ]{[}{]}{,}{n}{}{,}%

\usepackage{cleveref}
\crefname{subsection}{subsection}{subsections}
\crefname{lemma}{lemma}{lemmas}
\crefname{property}{property}{properties}
\crefname{table}{table}{tables}
\crefname{assumption}{assumption}{assumptions}

\newcommand{\BlackBox}{\rule{1.5ex}{1.5ex}}  
\newenvironment{prf}{\par\noindent{\bf Proof\ }}{\hfill\BlackBox\\[2mm]}
\newenvironment{prf1}{}{\hfill\BlackBox\\[2mm]}

\title{Online Estimation and Optimization \\of Utility-Based Shortfall Risk}

\author{
	Vishwajit Hegde\\
	{\normalsize Indian Institute of Technology Madras} \\
	{\normalsize \texttt{me17b039@smail.iitm.ac.in}}
	\and
	Arvind Menon\\
	{\normalsize Indian Institute of Technology Madras} \\
	{\normalsize \texttt{ep17b017@smail.iitm.ac.in}}
	\and 
	Prashanth L. A.\\
	{\normalsize Indian Institute of Technology Madras}\\
	{\normalsize \texttt{prashla@cse.iitm.ac.in}}
	\and
	Krishna Jagannathan\\
	{\normalsize Indian Institute of Technology Madras} \\
	{\normalsize \texttt{krishnaj@ee.iitm.ac.in}}
}
\date{}
\begin{document}
	
	\maketitle

\begin{abstract}
Utility-Based Shortfall Risk (UBSR) is a risk metric that is increasingly popular in financial applications, owing to certain desirable properties that it enjoys. We consider the problem of estimating UBSR in a recursive setting, where samples from the underlying loss distribution are available one-at-a-time. We cast the UBSR estimation problem as a root finding problem, and propose stochastic approximation-based estimations schemes. We derive non-asymptotic bounds on the estimation error in the number of samples. We also consider the problem of UBSR optimization within a parameterized class of random variables. We propose a stochastic gradient descent based algorithm for UBSR optimization, and derive non-asymptotic bounds on its convergence.
\end{abstract}

\section{Introduction}
\label{sec:intro}
In several financial applications, it is necessary to understand risk sensitivity while maximizing the returns. Several risk measures have been studied in the literature, e.g., mean-variance, Value at Risk (VaR), Conditional Value at Risk (CVaR), distorted risk measure, and prospect theory.  In \cite{artzner1999coherent}, the authors consider four properties as desirable for a risk measure, namely positive homogeneity, translation invariance, sub-additivity, and monotonicity. They define a risk measure as being \emph{coherent} if it possesses the aforementioned properties. In a related development, in \cite{follmer2002convex}, the authors chose to relax the sub-additivity and positive homogeneity requirements of a coherent risk measure, and instead impose a convexity condition on the underlying risk measure. Such a relaxation is justified in practical contexts where the risk is a non-linear function of the underlying random variable (e.g., a financial position).

CVaR is a popular risk measure that come under the umbrella of coherent risk measures. Utility-based shortfall risk (UBSR) \cite{follmer2002convex} is a risk measure that is closely related to CVaR, and one that belongs to the class of convex risk measures. UBSR as a risk measure is preferable over CVaR for a few reasons: (i) Unlike CVaR, UBSR is invariant under randomization;  (ii) UBSR involves a utility function that can be chosen to encode the risk associated with each value the random variable (r.v.) $X$ takes, while CVaR is concerned only with values of $X$ beyond a certain quantile; and (iii) UBSR is an \textit{elicitable} risk measure, i.e., there exists a scoring function, whose expectation is minimized by the risk measure, see \cite[Section 8]{follmer2015}. We refer the reader to \cite[Chapter 4]{follmer2016stochastic} for a detailed introduction to UBSR and related risk measures, and to \cite{follmer2015} for a survey on axiomatic approaches to risk metrics in monetary applications.

In real-world scenarios, the distribution of the underlying r.v. is seldom available in a closed form. Instead, one can obtain samples, which are used to estimate the chosen risk measure. Risk estimation has received a lot of attention in the recent past, cf. \cite{kagrecha2019distribution,pandey2021estimation,thomas2019concentration,dunkel2010stochastic,bhat2019concentration,prashanth2019concentration,prashanth2016cumulative,wang2010deviation,brown2007large,williamson2020cvarConc,lee2020oceConc}, with CVaR being the dominant choice for the risk measure.   

In this paper, we focus on recursive estimation of UBSR, in a setting where data arrives in an online fashion. Estimation of UBSR has immediate applications in financial portfolio optimization, cf. \cite{zhaolin2016convex}. Stochastic approximation \cite{robbins1951stochastic,borkar2008stochastic} is a procedure that is well-suited for the purpose of online estimation. In the context of UBSR estimation, our main contribution is the \emph{non-asymptotic} analysis of a stochastic approximation-based estimation scheme. We cast the estimation of UBSR  as a stochastic root finding problem, and derive `finite-sample' bounds for this scheme. Our analysis assumes that the underlying objective satisfies a monotonicity condition. If the monotonicity parameter is known and is used in setting the step-size,  the algorithm results in an $O(1/n)$ rate of mean-squared error decay. We also develop another variant that employs a universal step-size, and results in a $O(1/n^\alpha)$ rate, for any $0<\alpha<1$.
In addition, we also obtain a `high probability' result for the concentration of the estimation error.

Moving beyond UBSR estimation, we also consider the problem of optimizing UBSR within a parameterized class of random variables. The motivation for this problem lies in understanding the risk sensitivity in a portfolio management application \cite{rockafellar2000optimization,zhaolin2016convex}. Specifically, an investor could choose to distribute his/her capital among different assets, and the decision parameter governing the capital distribution is to be optimized to decide the best allocation. The utility function that goes into the definition of UBSR would encode the investor's risk preference, and the goal is to find the best decision parameter to minimize the risk, as quantified by UBSR.

For the problem of UBSR optimization, we propose a stochastic gradient algorithm, and derive non-asymptotic bounds on its performance. Stochastic gradient (SG) methods have a long history, and non-asymptotic analysis of such schemes has garnered a lot of attention over the last decade, see \cite{bottou2018optimization} for a survey. 
Unlike in a classic SG setting, the UBSR optimization problem involves \emph{biased}  measurements of both the UBSR value, as well as its derivative, which presents some technical challenges. Specifically, the UBSR estimation scheme is biased, in the sense that the estimation error does not have zero expectation. This is unlike in the classical SG settings, where the estimation error is assumed to be zero mean. In our setting, even though the estimation error is not zero-mean, the error can be reduced by increasing the batch size used for estimation.
For the purpose of gradient estimation, we leverage the UBSR sensitivity formula derived in \cite{zhaolin2016convex}, and use a natural estimator of this quantity based on independent and identically distributed (i.i.d.) samples. By controlling the batch size, we derive a $O(1/n)$ rate for the SG algorithm's mean-squared error to optimize the UBSR under a strongly convex objective. For the case of a convex objective, we obtain a $O(1/n)$ convergence rate for the objective function by employing a geometric phase-wise step size reduction scheme from \cite{jain2019making}.

\textit{Related work.} Stochastic approximation has been explored in the context of CVaR estimation in \cite{bardou2009cvar,bercu2020stochastic,costa2021non}. Recursive estimation of quantiles, variances and medians has been considered earlier in \cite{cardot2017online,cardot2013efficient,godichon2016estimating,costa2021non}. UBSR was introduced in \cite{follmer2002convex}, and non-recursive estimation schemes for UBSR were proposed in \cite{zhaolin2016convex}. A paper closely related to our work from UBSR estimation viewpoint is \cite{dunkel2010stochastic}, which uses a recursive estimation technique. The authors establish asymptotic convergence of their algorithm, and a `central limit theorem'  showing the asymptotic Gaussianity of the scaled estimation error. In contrast, we establish \emph{non-asymptotic}, i.e., finite-sample bounds for the performance of our recursive estimation method, under similar technical assumptions as \cite{dunkel2010stochastic,zhaolin2016convex}. \cite{duchi2012finite,balasubramanian2018zeroth} consider finite-sample analysis of zeroth-order stochastic approximation, but they assume zero-mean noise on the function measurements, which is not the case for UBSR optimization considered here. Other related papers include \cite{bhavsar2021nonasymptotic,pasupathy2018sampling} which consider stochastic approximation of an abstract objective function where the function measurements are biased, and the bias can be controlled through a batch size. In a recent paper  \cite{la2020concentration}, the authors use the estimation scheme from \cite{zhaolin2016convex} to establish concentration inequalities for UBSR estimation. A more detailed comparison to other related papers appears in Section~\ref{sec:comparison}.

The rest of the paper is organized as follows: 
In Section~\ref{sec:prelims}, we define the notion of UBSR for a general random variable, and in Section \ref{sec:pb}, we formulate the estimation as well as optimization problems under a UBSR objective.
In Section~\ref{sec:sr-est}, we
describe the stochastic approximation-based scheme for estimating the UBSR of a random variable, and present concentration bounds for this estimation scheme. In Section \ref{sec:sr-opt}, we present a stochastic gradient  algorithm for optimizing the UBSR in a parameterized class of random variables, and present a non-asymptotic bound that quantifies the convergence rate of this algorithm.  
In Sections \ref{sec:proofs-est}--\ref{sec:proofs-opt}, we provide proofs of the non-asymptotic bounds for UBSR estimation and optimization.
Finally, in Section~\ref{sec:conclusions}, we provide our concluding remarks.

\section{Utility-based shortfall risk}
\label{sec:prelims}
Let $X$ be a random variable, and $\ell(\cdot)$ be a convex loss function. Let $\lambda$ be a pre-specified ``risk-level'' parameter that lies in the interior of the range of $\ell$. We first define an acceptance set as follows:
\begin{align}
	\label{acc_set}
	\mathcal{A} \vcentcolon= \{ X \in L^\infty : \E[\ell(-X)] \leq \lambda \},
\end{align}
where $L^\infty$ represents the set of bounded random variables. 

Using the acceptance set, the utility-based shortfall risk (UBSR) $\sr(X)$ is defined by \cite{follmer2002convex}
\begin{align}
	\sr(X) \vcentcolon= \inf \{ t\in \mathcal{R} : t + X \in \mathcal{A} \}.
\end{align}
For notational convenience, we have made the dependence of UBSR $\sr(X)$ on the loss function $\ell$ implicit.
Intuitively, if $X$ represents a financial position, then $\sr(X)$ denotes the minimum cash that has to be added to $X$ so that it falls into the acceptable set $\mathcal{A}$. For a precise characterization of the relationship risk measures with convex acceptance sets and UBSR, the reader is referred to \cite{weber2006}.


UBSR is a particular example of a convex risk measure \cite{follmer2002convex}, which is a generalization of a coherent risk measure \cite{artzner1999coherent}. In particular, a coherent risk measure satisfies sub-additivity and positive-homogeneity, and these two properties readily imply convexity. 

As a risk measure, UBSR is preferable over the popular Value-at-Risk (VaR), owing to the fact that UBSR is convex. Another closely related risk measure is CVaR (Conditional Value at Risk), which is a coherent risk measure. UBSR has a few advantages over CVaR, namely\begin{itemize} \item[(i)] Unlike CVaR, UBSR is invariant under randomization. Specifically, suppose that $X_1,X_2$ are both acceptable (i.e., $\sr(X_i)\le 0, i=1,2.$), and that we use an independent Bernoulli r.v. to choose between $X_1$ and $X_2$, then we still have an acceptable financial position.  \item[(ii)] UBSR involves an loss function that can be chosen to encode the risk associated with each value the r.v. $X$ takes, while CVaR is concerned with values of $X$ beyond VaR at a pre-specified level $\alpha$. For a loss r.v. $X$ in a financial application, it makes sense to associate more risk with larger losses, and this can be encoded using, for example, an exponential loss function. On the other hand, CVaR considers all losses beyond a certain threshold equally. \item[(iii)] UBSR is an \textit{elicitable} risk measure, i.e., there exists a scoring function, whose expectation is minimized by the risk measure, see \cite[Section 8]{follmer2015}. It can be shown that VaR is elicitable, while CVaR is not. Moreover, UBSR is the only convex risk measure that is elicitable \cite[Theorem 8.6]{follmer2015}, while expectiles are the only coherent risk measure that are elicitable \cite{bellini2015,delbaen2016}.\end{itemize}
UBSR has been used for credit risk management under the Normal Copula Model \cite{gupton1997creditmetrics}, which is the foundation of the \textit{CreditMetrics} industry model; see also \cite{application,dunkel2010stochastic,zhaolin2016convex} for usage of UBSR in the context of portfolio optimization.

We now present two examples for the loss function.
\begin{example}\label{ex:exploss}
	The exponential loss function defined as follows: $\ell(x) = \exp(\beta x)$. The UBSR for this loss function is closely related to the relative entropy. More precisely, 
	\[\sr(X) = \frac{1}{\beta} \left( \log \E\left(\exp(-\beta X)\right) - \log \lambda \right).\]
	Thus, minimizing UBSR in this case is equivalent to an entropy minimization.
\end{example}

\begin{example}\label{ex:powerloss}
	With $p>1$, let	$\ell(x) = \begin{cases} \frac{1}{p} x^p, & x\ge 0,\\
		0 & \textrm{ otherwise.}\end{cases}$
\end{example}
The reader is referred to \cite[Section 4.9]{follmer2016stochastic} for a detailed discussion of these sample loss functions from a convex analysis perspective. In this paper, we focus on a statistical perspective --- specifically, of estimating and optimizing UBSR from `streaming' samples.

\section{Problem formulation}
\label{sec:pb}
In this paper, we focus on two problems concerning shortfall risk, namely (i) UBSR estimation, and (ii) UBSR optimization within a parameterized family of distributions. We define these two problems below.

Define the function \begin{align}
	g(t) \vcentcolon= \E[\ell(-X - t)] - \lambda.
	\label{eq:g-def}
\end{align}
We make the following assumption on the function $g$ defined above.
\begin{assumption}
	\label{ass:1}
	There exists $t_l, t_u$ s.t.  $g(t_{l}) > 0$ and $ g(t_{u}) < 0$. 
\end{assumption}
Under the above assumption, it can be shown using convexity and monotonicity of the loss function $\ell(\cdot)$ that $\sr(X)$ is finite, and also the unique root of the function $g$, i.e.,
the solution $t^*$ that satisfies $g(t^*)=0$ coincides with $\sr(X)$, cf. \cite[Proposition 2.3]{dunkel2010stochastic}.
Thus, the problem of UBSR estimation, i.e, estimating $\sr(X)$ of a r.v. $X,$ can be cast as a root finding problem. We consider a setting where the expectation in the definition of $g(\cdot)$ cannot be explicitly evaluated. Instead, we have access to independent and identically distributed (i.i.d.) samples from the distribution of $X,$ and we use a stochastic root-finding scheme for the UBSR estimation. 

	\emph{UBSR Estimation problem:} Find the unique root of $g(t)$ in \eqref{eq:g-def}, using i.i.d. samples from the distribution of $X.$
	
	Next, we define the the problem of UBSR optimization. Suppose that $X$ belongs to a parameterized family of distributions $\{X(\theta):\ \theta\in \Theta\},$ where $\Theta$ is a compact and convex subset of $\R.$ 
	
	\emph{UBSR Optimization problem:} For the parametrized family of distributions defined above, find
	\begin{align}
		\theta^* \in \argmin_{\theta \in \Theta} \sr(X(\theta)).
\end{align}
For the sake of simplicity, we focus on the case of a scalar parameter $\theta$. Again, assuming that we have access to samples from the distribution of $X,$ we use a stochastic gradient descent technique for SR optimization.


\section{UBSR estimation}
\label{sec:sr-est}
In this section, we propose and analyze a recursive estimation scheme for UBSR, assuming access to i.i.d. samples from the distribution of $X$ in an online fashion.

Stochastic approximation \cite{borkar2008stochastic} is a class of algorithms for solving stochastic root-finding problems; see \cite{pasupathykim2011} for a specialized survey on stochastic root finding. UBSR estimation can be viewed as a root-finding problem, since one has to find a $t^*$ satisfying $g(t^{*}) = 0$, or
$\E[\ell(-X - t^{*})] = \lambda$.
For this problem, \cite{dunkel2010stochastic} proposed a stochastic approximation scheme for estimating UBSR, assuming access to a quantile oracle. In practical applications, it may not be realistic to assume sample access from the quantile function of the underlying distribution. In contrast, we propose a simple stochastic approximation scheme that estimates UBSR using i.i.d. samples from the distribution of the r.v. $X$. Moreover, in \cite{dunkel2010stochastic}, the authors perform an asymptotic convergence analysis, while we derive non-asymptotic bounds for UBSR estimation. 

We propose a method to incrementally estimate UBSR using each additional sample. Specifically, we use the following update iteration:
\begin{align}
	t_{n} & = \Gamma(t_{n-1} + a_{n}\left ( \hat{g}(t_{n-1}) \right)),\label{eq:sr-est-update}
\end{align}
where 
$ \hat{g}(t) = \ell(\xi_{n} - t) - \lambda$ is an estimate of $g(t)$ using an i.i.d. sequence $\{\xi_i\}$ from the distribution of $-X$,
and  $\Gamma$ is a projection operator defined by $\Gamma(x)=\min(\max(t_l,x),t_u)$. Such a projection operator has been used in the context
of UBSR estimation earlier, cf. \cite{dunkel2010stochastic}.

One could estimate UBSR for a fixed set of samples using either a stochastic root-finding recursion such as \eqref{eq:sr-est-update} above, or perform a sample-average approximation using a binary search.
	In \cite{frikha2014}, the  authors study UBSR estimation in a specialized stochastic control framework for portfolio optimization with UBSR as the objective. The estimation algorithm that they consider for UBSR estimation is similar to the algorithm that we propose, except that they incorporate a normalization factor. However, their results are asymptotic in nature, while we derive finite time bounds for both UBSR estimation and optimization.
In \cite{zhaolin2016convex}, the authors analyze a sample-average approximation to UBSR to provide asymptotic convergence/rate guarantees. In contrast, using the stochastic root-finding approach, we establish non-asymptotic bounds as well. We prefer the root-finding approach due to its iterative nature, which would make it more widely applicable in machine learning applications with streaming data (e.g.,in multi-armed bandits, reinforcement learning with policy gradients etc.). The concentration and expectation bounds that we derive for UBSR estimation below are relevant in such applications.

\subsection*{Main results}    
In addition to \Cref{ass:1}, we make the following assumptions for the bounds on UBSR estimation.
\begin{assumption}
	\label{ass:Xbdd}
	$|X|\le \B$ almost surely (a.s.). 
\end{assumption}
\begin{assumption}
	\label{ass:2}
	There exists a $\mu_1, L_1>0$ s.t. $-L_1 \leq -\ell'(-t) \leq -\mu_1,$ for all $t \in [t_{l},t_{u}]$.
\end{assumption}
\begin{assumption}
	\label{ass:3} Let $\varepsilon_{n}=\hat{g}(t_{n}) - g(t_{n})$. 
	Let $\F_{n}$ denote the $\sigma$-field generated by $\{t_k, k \le n\}$.
		There exists a $\sigma>0$ such that
		$\E[\varepsilon_{n}^{2}\mid \F_{n}] \leq \sigma^{2}$  almost surely (a.s.) for all $n\ge 1$.
\end{assumption}
Previous works on UBSR estimation (cf. \cite{zhaolin2016convex}, \cite{dunkel2010stochastic}) make similar assumptions.
In \cite{follmer2002convex}, the definition of UBSR is for a r.v. that is bounded, justifying \Cref{ass:Xbdd}.
In \Cref{ass:2}, we assume $\ell$ is differentiable, which implies that $g$ is differentiable as well. Note that 
It is easy to see that the loss functions in Examples \ref{ex:exploss} and \ref{ex:powerloss} satisfy \Cref{ass:2}, as does the financial portfolio example in Section~\ref{sec:ubsr-opt-portfolio}.
Next, \Cref{ass:3} requires that the underlying noise variance is bounded:  a natural assumption in the context of an estimation problem. 



The first result below is a non-asymptotic bound on the estimation error $\E[(t_{n}-\sr(X))^{2}]$ for a stepsize choice that requires the knowledge of $\mu_1$ from \Cref{ass:2}. 
\begin{theorem}
	\label{thm:sr-est-muknown}
	Suppose \Crefrange{ass:1}{ass:3} hold. Setting the step size $a_{k} = \frac{c}{k}$ with $\frac{1}{2} < \mu_1 c $, we have
	\begin{align}
		\E[(t_{n}-\sr(X))^{2}] \leq \exp\left(\frac{L_1^2c^2\pi^2}{6}\right) \left[ \frac{(t_{0}-\sr(X))^{2}}{n^{2\mu_1 c}} + \frac{\sigma^{2}2^{2\mu_1 c}c^{2}}{(2\mu_1 c - 1)n}\right].
		\label{eq:sr-est-bd1}
	\end{align}
\end{theorem}
\begin{prf}
	See Section \ref{appendix:proof-sr-est1}.
\end{prf}
\begin{remark}
	The non-asymptotic bounds for UBSR estimation in this section, and for UBSR optimization in Section \ref{sec:sr-opt} are stated in a form similar to those for other stochastic approximation schemes, cf. \cite{frikha2012concentration}. 
	The first term on the RHS in the bound above concerns the initial error, i.e., the rate at which the algorithm `forgets' the starting point $t_1$. The second term relates to the noise variance in UBSR estimation. From the bound above, together with the fact that $\frac{1}{2} < \mu_1 c$, it is apparent that the initial error is forgotten faster than the error due to the noise. On a different note, from the bound in \eqref{eq:sr-est-bd1}, it is apparent that $\E[(t_{n}-\sr(X))^2]$ scales linearly with the reciprocal of the monotonicity parameter $\mu_1$.  
\end{remark}
\begin{remark}
	In \cite{dunkel2010stochastic}, the authors propose a stochastic approximation scheme that uses the quantile function of $X$.
	Letting $\tilde t_n$ denote their stochastic approximation iterate,
	they establish that $n^{1/2}(\tilde t_n-\sr(X))$ is asymptotically normal, say $N(0,\zeta^2)$ for a step-size choice that requires the knowledge of $g'(\sr(X))$. Under mild regularity conditions (cf. \cite{gerencser1999convergence}), the asymptotic normality result implies $n E(\tilde t_n-\sr(X))^2$ converges to a constant that depends on $\zeta^2$. The result we derived in Theorem \ref{thm:sr-est-muknown} holds for all $n$, and does not require access to the quantile function of $X$. Nevertheless, our $O(1/n)$ non-asymptotic bound is consistent with the asymptotic convergence rate  of \cite{dunkel2010stochastic}.
\end{remark}

Next, we present a high probability bound for the SR estimation algorithm in \eqref{eq:sr-est-update}.
\begin{theorem}
	\label{thm:sr-est-muknown-hpb}
	Suppose  \Crefrange{ass:1}{ass:3} hold. Set the step size $a_{k} = \frac{c}{k}$ with $\frac{1}{2} < \mu_1 c $. Then, for any $\delta \in (0,1)$, the following bound holds w.p. at least $(1-\delta)$ :
	\begin{align}
		|t_n-\sr(X)|&\le \sqrt{\frac{\log\left(1/\delta\right)}{C_1 n}} + 
		\exp\left(\frac{L_1^2c^2\pi^2}{12}\right)\left[
		\frac{\E[|t_1-t^*|]}{n^{\mu_1 c}} + \frac{c \sigma 2^{2\mu_1 c}}{\sqrt{(2\mu_1 c - 1)}\sqrt{n}}\right]
		, \label{eq:hpb-sr-est}
	\end{align}
	where $C_1 = \frac{(2\mu_1 c - 1)}{2^{4\mu_1 c + 4}c^{2} L_1^2\B^2}\exp\left(-\frac{L_1^2c^2\pi^2}{6}\right)$.
\end{theorem}
\begin{prf}
	See Section \ref{appendix:proof-sr-est1-hpb}.    
\end{prf}
The two results presented above required the knowledge of the monotonicity parameter $\mu_1$, which is typically unknown in a risk-sensitive learning setting.
We now present a bound on the UBSR estimation error under a universal stepsize, i.e., one which does not require the knowledge of $\mu_1$.  
\begin{theorem}
	\label{thm:sr-est-muuniv}
	Suppose \Crefrange{ass:1}{ass:3} hold. Choose an $n_0$ such that $a_{n_0} L_1^2 < \mu_1$. Then, we have the following bounds for two different step sizes:\\[1ex]
	\textbf{Case I:} Set $a_k = \frac{c}{k}$. Then, for any $n\ge n_0$, 
	\begin{align*}
		\E[(t_{n}-\sr(X))^{2}] &\leq 
		C(n_0)\left(\E[(t_{0}-\sr(X))^{2}] + \sigma^{2}\frac{\pi^2}{6}\right)\frac{1}{{n}^{\mu_1 c}}
		+\K_1(n),
	\end{align*}
	where 
	\begin{align*}
		C(n_0) &= (1 + cL_1)^{2n_0} (n_0+1)^{\mu_1 c}\textrm{ and}\\[0.5ex]
		\K_1(n) &= \begin{cases} O\left(1/n^{\mu_1 c}\right) &\textrm{ if } \mu_1 c  < 1\\
			O\left(\log n/n\right) &\textrm{ if } \mu_1 c  = 1, \\ 
			O\left(1/n\right) & \textrm{ if } \mu_1 c  > 1.\end{cases}
	\end{align*}
	\textbf{Case II:} Set $a_k = \frac{c}{k^{\alpha}}$ for some $\alpha \in (0,1)$. Then, for any $n\ge n_0$,
	\begin{align*}
		\E[(t_{n}-\sr(X))^{2}] &\leq
		\widetilde C(n_0)\left(\E[(t_{0}-\sr(X))^{2}] + \sigma^{2}c^2 n_0\right)\exp\left(-\frac{\mu_1 cn^{1-\alpha}}{1-\alpha} \right)+ \\
		&\frac{2\sigma^{2}c^2 (\mu_1 c)^{\frac{\alpha}{1-\alpha}}}{(1-\alpha)n^\alpha},
	\end{align*}
	where $\widetilde C(n_0)=(1 + cL_1)^{2n_0}  \exp\left(\frac{\mu_1 cn_0^{1-\alpha}}{1-\alpha} \right).$
\end{theorem}
\begin{prf}
	The proof proceeds by dividing the analysis into two parts about $n_0$.	See Section \ref{appendix:proof-sr-est2} for the details.
\end{prf}
\begin{table}[h]
	\centering
	
	\begin{tabular}{c|c}
		\toprule
		Step size & Sample complexity \\\hline
		$a_{k} = \frac{c}{k}$ with $\frac{1}{2} < \mu_1 c $ & $O\left(\frac{1}{\epsilon^2}\right)$   \\\hline
		$a_{k} = \frac{c}{k^\alpha}$ & $O\left(\frac{1}{\epsilon^{2/\alpha}}\right)$   \\\hline
	\end{tabular}
	\caption{Sample complexity of UBSR estimation for achieving $\E[\left|t_{n}-\sr(X)\right|] < \epsilon$}
	\label{tab:splcases}
\end{table}

Table \ref{tab:splcases} summarizes the sample complexity of achieving $\epsilon$-accuracy for UBSR estimation under two different step sizes. A few remarks are in order.
\begin{remark}
	\label{rem:r1}
	For Case I, the estimation error can decay as $1/n$ if $c$ is chosen such that $\mu_1 c>1.$ However, if $\mu_1$ is not known, such a choice may not be feasible. Indeed, the error can decay much slower if $c$ is such that $\mu_1 c$ is much smaller than $1.$ For Case II above, the estimation error decays as $1/n^\alpha$ where $\alpha$ can be chosen arbitrarily close to 1 when deciding the step size, and this choice does not depend on $\mu_1$. However, as $\alpha$ approaches 1, the first term grows in an unbounded manner.
	An advantage with the larger stepsize $c/k^\alpha$ in Case II is that the initial error is forgotten exponentially fast, while the corresponding rate is $1/n^{\mu_1 c}$ for the stepsize $c/k$.
\end{remark}
\begin{remark}
	The step size in Case II above is typically used in conjunction with \emph{iterate averaging} \cite{polyak1992acceleration,ruppert1991stochastic}. We can also use iterate averaging in this setting, but we can show that it does not improve the error decay rate derived for Case II without employing iterate averaging. Using iterate averaging, a mean-squared parameter error bound  of $O(1/n)$ has been shown for a stochastic gradient algorithm, when the underlying objective is strongly convex in \cite{moulines2011non}. Although our bound is weaker in comparison (when $\mu_1$ is unknown), our algorithm applies for the last iterate, which is the preferred practical choice \cite{jain2019making}. Further, unlike \cite{moulines2011non}, we derive high-probability bounds on the UBSR estimation error. 
\end{remark}

\begin{remark}
	The authors in \cite{dunkel2010stochastic} analyze a iterate-averaged variant of the SR estimation algorithm \eqref{eq:sr-est-update}, while assuming the knowledge of $g'(\sr(X))$ for setting the step-size constant $c$. The rate they derive under this assumption is $O(1/n)$ asymptotically. In comparison, our analysis is for a universal step-size, and we obtain a non-asymptotic bound of $O(1/n^\alpha)$, for $\alpha\in (0,1)$. In practice, the knowledge of $g'(\sr(X))$ is seldom available, motivating the universal step-size choice. The rate we derive in this case is comparable to the one obtained in \cite{fathi2013transport} for general stochastic approximation schemes. 
\end{remark} 

The final result on UBSR estimation is a high probability bound for a universal stepsize choice.
\begin{theorem}
	\label{thm:sr-est-muuniv-hpb}
	Suppose \Crefrange{ass:1}{ass:3} hold. Set the step size $a_{k} = \frac{c}{k^\alpha}$ with $\alpha\in (0,1)$, and choose an $n_0$ such that $L_1^2 a_{n_0} < \mu_1$. Then, for any $\delta \in (0,1)$, and for any $n\ge n_0$, we have the following bound w.p. at least $(1-\delta)$:
	\begin{align}
		&|t_n-\sr(X)| \le C_2\exp\left(-\frac{\mu_1 cn^{1-\alpha}}{2(1-\alpha)}\right) +\frac{C_3}{n^{\alpha/2}},
		\label{eq:hpb3}
	\end{align}
	where 
	\begin{align*}
		C_2 &= 8 L_1 \B \sqrt{\frac{\log\left(1/\delta\right)(1+c^2L_1^2)^{n_0+1}c^2)}{c^2L_1^2 }} +\sqrt{\widetilde C(n_0)\left(\E[(t_{0}-\sr(X))^{2}] + \sigma^{2}c^2 n_0\right)} ,\textrm{ and}\\
		C_3 &= \left(8 L_1 \B\sqrt{\frac{\log\left(1/\delta\right)2(\mu_1 c)^{\frac{\alpha}{1-\alpha}}c^2 }{(1-\alpha)}}+ \sqrt{\frac{\sigma^{2}2(2\mu_1 c)^{\frac{\alpha}{1-\alpha}}c^2}{(1-\alpha)}}\right).    
	\end{align*}
	In the above, $\mu_1, \sigma^2$, and  $L_1$ are specified in Assumptions \ref{ass:2} and \ref{ass:3}, while the constant $\widetilde C(n_0)$ is as defined in Theorem~\ref{thm:sr-est-muuniv}. 
\end{theorem}
\begin{prf}
	See Section \ref{appendix:proof-sr-est2-hpb}.
\end{prf}

In the result above, we have chosen the stepsize to be $c/k^\alpha$ as choosing $c/k$ does not guarantee a $O(1/n)$ rate (see Remark \ref{rem:r1}).

\section{UBSR Optimization}
\label{sec:sr-opt}
Recall  the UBSR optimization problem:
\begin{align}
	\textrm{Find } \theta^* \in \argmin_{\theta \in [\theta_l,\theta_u]} \sr(X(\theta)),\label{eq:sr-opt-pb}
\end{align}
for some $0<\theta_l<\theta_u<\infty$.
	Recall that we operate in a risk-sensitive learning framework, i.e., we do not have direct access to UBSR $\sr(\theta)$ and its derivative $\frac{d \sr(\theta)}{d\theta}$, for any $\theta$. Instead, we can obtain samples of the underlying r.v. $X(\theta)$ corresponding to any parameter $\theta$. 

We  now devise a stochastic gradient algorithm that aims to solve the problem \eqref{eq:sr-opt-pb} using  the following update iteration:
	\begin{align}
		\theta_{k+1} & = \theta_{k} - b_{k} h_m'(\theta_k),\label{eq:sr-gd-update}
	\end{align}
	where $b_k$ is a step-size parameter, 
	$m$ is the number of i.i.d. samples from the distribution of $X(\theta_k)$, and  $h_m'(\theta_k)$ is an estimate of $\frac{d \sr(\theta)}{d\theta}$. 
	The reader is referred to 
	Algorithm \ref{alg:ubsr-sg} for  the pseudocode. 
	
	In the next section, we describe the UBSR derivative estimation scheme used in Algorithm \ref{alg:ubsr-sg}, and subsequently present non-asymptotic bounds for the iterate governed by \eqref{eq:sr-gd-update}.

\subsection{Estimation of UBSR derivative}
In \cite{zhaolin2016convex}, the authors derive an expression for the derivative of $\sr(X(\theta))$ under the following assumptions: Let $\xi = -X$.
	\begin{assumption}
		\label{ass:lprimelowerbound} The loss function $\ell(\cdot)$ is twice differentiable, and for any $\theta \in [\theta_l,\theta_u]$, $\ell'(\xi(\theta) - \sr(\theta)) > \eta$ w.p. $1$. 
	\end{assumption}
	\begin{assumption}
		\label{ass:partials-exist}
		The partial derivatives $\partial \ell(\xi(\theta - t(\theta))))/\partial \theta$ and $\partial \ell(\xi(\theta) - t(\theta))/\partial t $ exist w.p. $1$. 
	\end{assumption}
	Let $g(\theta,t) \vcentcolon = \E[\ell(-X(\theta) - t)] - \lambda$. As discussed earlier, $g(\theta, \sr(\theta))=0$.  
	Using \Cref{ass:partials-exist} and applying the dominated convergence theorem, it can be shown that
	\begin{align*}
		\partial g(\theta, \sr(\theta))/\partial \theta = - \E[ l'(\xi(\theta) - \sr(\theta)))\xi'], \quad \partial g(\theta, \sr(\theta))/\partial t = \E[l'(\xi(\theta) - \sr(\theta))].
	\end{align*}
	Next, using implicit function theorem, the UBSR derivative can be expressed as follows:
	\begin{align*}
		\frac{d\sr(\theta)}{d\theta} = -\frac{\partial g(\theta, \sr(\theta))/\partial \theta}{\partial g(\theta, \sr(\theta))/\partial t} &= \frac{\E[ l'(\xi(\theta) - \sr(\theta)))\xi']}{\E[l'(\xi(\theta) - \sr(\theta))]}
	\end{align*}
	Letting $A(\theta)\triangleq \E[ (\ell'(\xi(\theta) - \sr(\theta)))\xi'(\theta)],$  and  $B(\theta)\triangleq \E[\ell'(\xi(\theta) - \sr(\theta))]$, we have
	\begin{align}
		& \frac{d \sr(\theta)}{d\theta} =  \frac{A(\theta)}{B(\theta)}.\label{eq:sr-derivative}
\end{align} 

\begin{algorithm}[t]
	\caption{Stochastic gradient algorithm for UBSR optimization}\label{alg:ubsr-sg}
	\SetKwInOut{Init}{Input}\SetKwInOut{Output}{Output}
	\Init{Initial points $\theta_0, t_0$,  step sizes $\{ a_k, b_k \}$, batch sizes $\{ m_k \}$, and an iteration limit $n \ge 1$.}
	\For{$k = 1, \ldots, n$}{
		\tcc{Monte Carlo simulation}
		Obtain $\{\xi_1,\ldots,\xi_m\}$ and $\{\tilde\xi_1,\ldots,\tilde\xi_m\}$ samples from the distribution of $-X(\theta_k)$\;
		\tcc{UBSR estimation}
		Run $m$ iterations of \eqref{eq:sr-est-update} using  the samples $\{\tilde\xi_1,\ldots,\tilde\xi_m\}$, and with initial value $t_0$\;
		Let $t_m(\theta_k)$ denote the UBSR estimate obtained above\;
		\tcc{Gradient estimation}
		\centerline{$h_m'(\theta_k) = \frac{A_m}{B_m}, \textrm{ where } A_{m} = \frac{1}{m}\sum\limits_{i=1}^{m}\ell'(\xi_i(\theta_k) - t_m(\theta_k))\xi'_i(\theta_k) ,\ B_{m} = \frac{1}{m}\sum\limits_{i=1}^{m}\ell'(\xi_i(\theta_k) - t_m(\theta_k))$.}
		\tcc{Update iteration}
		\centerline{$
			\theta_{k+1}  = \theta_{k} - b_{k} h_m'(\theta_k).$}
	}
	\Output{Parameter $\theta_n$} 
\end{algorithm}

We now present a scheme for estimating the UBSR derivative $\frac{d \sr(\theta)}{d\theta}$, for a given $\theta$.
	We use double sampling to form an estimate of the UBSR derivative. More precisely, 
	suppose we are given i.i.d. samples $\{\xi_1,\ldots,\xi_m\}$ and $\{\tilde\xi_1,\ldots,\tilde\xi_m\}$ from the distribution of $-X(\theta)$ for a given parameter $\theta$. 
	From the samples $\{\tilde\xi_1,\ldots,\tilde\xi_m\}$, we form an estimate $t_m(\theta)$ of $\sr(\theta)$  using \eqref{eq:sr-est-update}. In other words, $t_m(\theta)$ is estimate of $\sr(\theta)$, which is obtained by running \eqref{eq:sr-est-update} for $m$ iterations.
	Next, using the samples $\{\xi_1,\ldots,\xi_m\}$ and $t_m(\theta)$, we form an estimator $h'_m(\theta)$ of UBSR derivative as follows:
	\begin{align}
		h'_m(\theta) = \frac{A_m}{B_m}, \label{eq:sr-derivative-est}
	\end{align}
	where $A_{m}(\theta) = \frac{1}{m}\sum\limits_{i=1}^{m}\ell'(\xi_i(\theta) - t_m(\theta))\xi'_i(\theta) ,\ B_{m}(\theta) = \frac{1}{m}\sum\limits_{i=1}^{m}\ell'(\xi_i(\theta) - t_m(\theta))$.
	We shall see later (Lemma \ref{lem:srprime-consistent}) that ``independent double sampling'' above helps us bound the mean-squared error of the UBSR derivative estimator by avoiding cross-terms. 
	Notice that the estimate defined above is a ratio of estimates for the quantities $A(\theta)$ and $B(\theta)$, which are used in the expression \eqref{eq:sr-derivative} for $\frac{d \sr(\theta)}{d\theta}$.
	Further, $A_m(\theta)$ and $B_m(\theta)$ are \emph{biased} estimates of $A(\theta)$ and $B(\theta)$, since the UBSR estimate $t_m(\theta)$ is biased. 
	Hence, it is apparent that $h'_m(\theta)$ is a biased estimate of the UBSR derivative. 
	Even though our estimator $A_m/B_m$ for UBSR derivative is biased, we show in Lemma \ref{lem:srprime-consistent} that the mean-squared estimation error is of the order $O(1/m)$, which implies our estimator converges to the UBSR derivative as the number of samples $m$ tends to infinity.
	
	In a related work \cite{glynn2015unbiased}, the authors provide an estimator for a ratio of the form $\mathbb E[A]/\mathbb E[B]$, while assuming unbiased estimators for both  numerator and denominator. In our setting for UBSR optimization, we do not have unbiasedness owing to the fact that the true UBSR value appears in both $A$ and $B$, and UBSR is not directly available, while one can estimate the same. Moreover, the order of the MSE bound for the estimation scheme in \cite{glynn2015unbiased} is comparable to the $O(1/m)$ that we derive in Lemma \ref{lem:srprime-consistent}.

\paragraph{Assumptions.}
We make the following assumptions for analyzing the bias and variance of the UBSR derivative estimate \eqref{eq:sr-derivative-est}.
	Recall that $\xi = -X$.

	\begin{assumption}
		\label{ass:lossLipschitz}
		The loss function $\ell(\cdot)$ satisfies w.p. $1$
		\begin{align*}
			|\ell''(\xi(\theta)-t) | \leq L_2, \forall (\theta, t) \in [\theta_l,\theta_u] \times [t_l,t_u].
		\end{align*}
	\end{assumption}
	\begin{assumption}
		\label{ass:lossVariance}
		There exist $\varl, \tilde\varl >0$ such that
		\begin{align*}
			\variance(\ell'(\xi(\theta)-\sr(\theta))) \leq \varl^2, \variance(\ell'(\xi(\theta)-\sr(\theta))\xi'(\theta)) \leq \tilde\varl^2.
		\end{align*}
\end{assumption}

\begin{assumption}
		\label{ass:xiprimeBound}
		$\sup_{\theta \in [\theta_l,\theta_u]} |\xi'(\theta)|\le M_2$. 
\end{assumption}
As a consequence of \Cref{ass:lossVariance}, the quantity $\E\left[(\ell'(\xi(\theta) - \sr(\theta)))^2\right] $ is bounded above for all $\theta$. We shall denote this upper bound by $\beta_1$, i.e., 
	\begin{align}
		\E\left[(\ell'(\xi(\theta) - \sr(\theta)))^2\right] \le \beta_1 < \infty, \forall \theta \in [\theta_l,\theta_u].
		\label{eq:beta1}
\end{align}
We now discuss the motivation behind the assumptions listed above.
	Assumptions \ref{ass:lossLipschitz} and \ref{ass:xiprimeBound} ensure that the UBSR is a smooth function of $\theta$.  
	The variance bounds in \Cref{ass:lossVariance} allows us to derive a $O\left(1/m\right)$ bound on the mean-square error of the UBSR derivative estimator \eqref{eq:sr-derivative-est}. Similar assumptions have been made in \cite{zhaolin2016convex} in the context of an asymptotic normality result for the UBSR derivative estimate.
Note that the loss functions in Examples \ref{ex:exploss} and \ref{ex:powerloss} satisfy all the assumptions.

We now present bounds on the bias and the mean-square error of the UBSR derivative estimate \eqref{eq:sr-derivative-est}, where we use the estimation bound from Theorem \ref{thm:sr-est-muknown}, i.e., UBSR estimation using \eqref{eq:sr-est-update}.
	\begin{lemma}
		\label{lem:srprime-consistent}
		Suppose \Crefrange{ass:1}{ass:xiprimeBound} hold. Then for all $m\ge 1$, the UBSR derivative estimator \eqref{eq:sr-derivative-est} satisfies
		\begin{align}
			\E\left|h_m'(\theta) - \frac{d\sr(\theta)}{d\theta}\right| &\leq   \frac{C_4}{\sqrt{m}}, \textrm{~~ and ~~} 
			\E\left|h_m'(\theta) - \frac{d\sr(\theta)}{d\theta}\right|^2 \le \frac{C_5}{m}, \label{eq:ubsr-derivative-bd}
		\end{align}
		where 
		$C_4 = \frac{\sqrt{\beta_1}(K_1(m)M_2L_2 + \tilde\varl)  + \sqrt{\beta_1} M_2(K_1(m)L_2 + \varl)}{\eta^2}$ and
		$C_5 = \frac{2\beta_1(2 L_2^2 K_1(m)^2 + \varl^2) + 2\beta_1M_2^2(2 L_2^2 M_2^2 K_1(m)^2 + \tilde\varl^2)}{\eta^4}$, with 
		$$K_1(m)=   \exp\left(\frac{L_1^2c^2\pi^2}{12}\right) \left[ \frac{|t_{0}-\sr(X)|}{m^{\mu_1 c-\frac1{2}}} + \frac{2^{\mu_1 c}c\sigma}{(\mu_1 c - \frac1{2})}\right].$$
		Here the constants $L_1, L_2$, $L_3, M_0, M_2$ and $\eta$ are as specified in Assumptions \ref{ass:Xbdd} and \ref{ass:lprimelowerbound} -- \ref{ass:xiprimeBound}, while $\beta_1$ is specified in \eqref{eq:beta1}.
\end{lemma}
We remark that the bound on the mean-squared error of the UBSR derivative estimator cannot be improved except for constant factors. This is because the minimax lower bound for mean estimation matches the order of the bound obtained above, see \cite[Section 8.3]{duchi2023statistics}.
\begin{prf}
	See Section \ref{appendix:proof-lemma-opt2}. 
\end{prf}
Since $\mu_1 c>\frac{1}{2}$, it is easy to see that $K_1(m)=O(1)$. This factor $K_1(m)$ appeared in \eqref{eq:ubsr-derivative-bd} since we require $t_m(\theta)$ to be close to $\sr(\theta)$ for $h_m'(\theta)$ to be a good estimate of the UBSR derivative. Furthermore, the UBSR estimation bound from Theorem \ref{thm:sr-est-muknown} assumes the knowledge of the parameter $\mu_1$ (See \Cref{ass:2}). One can also claim a straightforward variation of Lemma \ref{lem:srprime-consistent} that employs the stepsize scheme from Theorem \ref{thm:sr-est-muuniv} to arrive at a bound of $O\left(\frac{1}{m^{\alpha/2}}\right)$, but with a stepsize that does not require knowledge of $\mu_1$. 

Under assumptions similar to those listed above, the authors in \cite{zhaolin2016convex} establish an asymptotic consistency as well as normality results. In contrast, we establish a result in the non-asymptotic regime, with an $O(\frac{1}{\sqrt{m}})$ that matches the aforementioned asymptotic rate.


\subsection{Non-asymptotic bounds for UBSR optimization} In this section, we derive non-asymptotic bounds for UBSR optimization using the biased derivative estimates given above. We treat the strongly convex case first and then generalize to any convex $\sr(\cdot)$ in the subsequent subsection. 


\subsubsection{Strongly convex case}
In this subsection, let us assume that the UBSR objective $\sr(\theta)$ is a strongly convex function, i.e.,
\begin{assumption}
	\label{ass:strongconvexity} For any $\theta \in [\theta_l,\theta_u]$, the function $h(\theta)=\sr(\theta)$ satisfies $h''(\theta) > \mu_2$, for some $\mu_2>0$.
\end{assumption}
We now present a non-asymptotic bound for the last iterate $\theta_n$ of the algorithm \eqref{eq:sr-gd-update} with gradient estimates formed using \eqref{eq:sr-derivative-est}. The batch size $m$ used for gradient estimation is kept constant in each iteration $k=1,\ldots,n$.
Using the results from Lemma \ref{lem:srprime-consistent} in conjunction with \Cref{ass:strongconvexity}, we present a bound on the error $\E[\|\theta_{n} - \theta^{*} \|^2]$ in the optimization parameter in the theorem below. 
\begin{theorem}
		\label{thm:sr-opt-muknown}
		Suppose \Crefrange{ass:1}{ass:strongconvexity} hold. Let $\theta^*$ denote the minimum of $\sr(\cdot)$. Set $b_k = b/k$ in \eqref{eq:sr-gd-update}, with $\mu_2 b>\frac{1}{2}$. For each iteration of \eqref{eq:sr-gd-update}, let $m$ denote the batch size used for computing the estimate \eqref{eq:sr-derivative-est} corresponding to the parameter $\theta_k$, $k=1,\ldots,n$.
		Then, for all $n\ge 1$, we have 
		\begin{align}
			\E[(\theta_{n} - \theta^*)^2]&\leq \exp\left(\frac{c^2L_4^2\pi^2}{6}\right)\left[\frac{3(\theta_{0} - \theta^*)^2}{n^{2\mu_2 b}} + \frac{C_6}{n} + \frac{C_7}{m}\right],\label{eq:sr-opt-bound}
		\end{align}
		where
		$L_4=\frac{2L_1L_2M_2^2 + 2L_1L_2B_1M_2 + L_1^2L_3}{\eta^2}$,
		$C_6 = \frac{3C_52^{2\mu_2b}b^2}{(2\mu_2b - 1)}$, and
		$  C_7 = \frac{3C_5 b^2 2^{5\mu_2 b}}{(\mu_2b)^2}$, with $C_5$ as defined in Lemma \ref{lem:srprime-consistent}.
\end{theorem}
\begin{prf}
	See Section \ref{appendix:proof-sr-opt1}.
\end{prf}
The batch size could be chosen as a function of the horizon $n$. Results in a similar spirit, i.e., where a stochastic gradient algorithm is run for $n$ iterations, and the parameters such as step-size and batch size are set as a function of $n$ are common in the literature, cf. \cite{ghadimi2013stochastic,balasubramanian2018zeroth}.

The first term in \eqref{eq:sr-opt-bound} represents the initial error, and it is forgotten at a rate faster than $O(1/n)$ since $\mu_2 c > 1/2$.
The overall rate for the algorithm would depend on the choice of the batch size $m$, and it is apparent that the error $ \E[\|\theta_{n} - \theta^*\|^2]$ does not vanish with a constant batch size. As in the case of Theorem \ref{thm:sr-est-muknown}, we observe that the error $ \E[\|\theta_{n} - \theta^*\|]$ has an inverse dependence on the strong convexity parameter $\mu_2$. 

We now present a straightforward corollary of the result in Theorem \ref{thm:sr-opt-muknown} with a batch size that ensures the error in the parameter vanishes asymptotically.
\begin{corollary}
	\label{cor:sr-opt}
	Under conditions of Theorem \ref{thm:sr-opt-muknown}, with $m = n^\rho$ for some $\rho \in (0,1]$, we have
	\begin{align*}
		\E[(\theta_{n} - \theta^*)^2]&\leq \exp\left(\frac{b^2L_4^2\pi^2}{6}\right)\left[\frac{3(\theta_{0} - \theta^*)^2}{n^{2\mu_2 b}} + \frac{C_6}{n} + \frac{C_7}{n^\rho}\right]= O\left(\frac{1}{n^\rho}\right).
	\end{align*}
\end{corollary}
A few remarks are in order.
\begin{remark}
		From the result in the corollary above, it is easy to see that the optimal choice of batch size is $m=\Theta(n)$, and this in turn ensures an $O\left(\frac{1}{n}\right)$ rate of convergence for the stochastic gradient algorithm \eqref{eq:sr-gd-update}. With a biased derivative estimation scheme in a slightly different context, the authors in \cite{atchade2014stochastic} show that an increasing batch size is necessary for the error of gradient descent type algorithm to vanish. Finally, the $O(1/n)$ bound in Theorem \ref{thm:sr-opt-muknown}, which is for a setting where gradient estimates are biased, matches the minimax complexity result for strongly convex optimization with a stochastic first order oracle, cf. \cite{agarwal2010scoLB}.  
\end{remark}

\begin{remark}
	In the result above, we have bounded the error $ \E[(\theta_{n} - \theta^*)^2]$ in the optimization parameter. Using \Cref{ass:lossLipschitz} and $m=\Theta(n)$, we can also bound the optimization error $ \E[\sr(\theta_{n})] - \sr(\theta^*)]$ using Corollary \ref{cor:sr-opt} as follows:
	\begin{align*}
		\E[\sr(\theta_{n})] - \sr(\theta^*) \le \frac{L_2}{2} \E[(\theta_{n} - \theta^*)^2] = O\left(\frac{1}{n}\right).
	\end{align*}
	For achieving this $O\left(\frac{1}{n}\right)$ rate, we used a batch size of $\Theta(n)$ in each iteration of \eqref{eq:sr-gd-update}, leading to a total sample complexity of $\Theta(n^2)$.
\end{remark}
\begin{remark}
	To understand the deviation from the non-asymptotic analysis of a regular stochastic gradient algorithm (cf. \cite{moulines2011non}), we provide a brief sketch of the proof of Theorem \ref{thm:sr-opt-muknown}.\\
	Letting $M_{k}  = \int\limits_{0}^{1} [h''(m \theta_{k} + (1-m)\theta^{*})]dm$, and $z_n=\theta_n -\theta^*$, we have
	\begin{align*}
		z_{k} & = z_{k-1}(1 -  b_kM_{k-1}) - b_k\varepsilon_{k-1},
	\end{align*}
	where $\varepsilon_{k} = h'_m(\theta_k)  - h'(\theta_k)).$\\
	Unlike the setting of \cite{moulines2011non}, the noise in derivative estimate $\varepsilon_{k}$ is biased, i.e., $\E[\varepsilon_{k}] \ne 0$. Now, unrolling the recursion above and taking expectations, we obtain
	\begin{align}
		\E\|z_{n}\|^2 &\leq 3 \E[\|z_0\|^2]\prod\limits_{k=1}^{n}(1 - b_kM_{k-1})^2 + 3 \E[\sum\limits_{k=1}^{n}[b_k \varepsilon_{k-1} \prod\limits_{j=k+1}^{n}(1 - a_{j}M_{j-1})]^2\label{eq:s123}\\
		\begin{split}
			&\leq 3\E[\|z_0\|^2]n^{-2\mu_2 c} + 3 \underbrace{\sum\limits_{k=1}^{n}\frac{c^2}{k^2} \E[\varepsilon_{k-1}^2] (\mathcal{P}_{k+1:n})^2}_\text{I} +\\ &  \quad 3\underbrace{\sum\limits_{k\neq l}^{n}b_ka_{l}\E[|\varepsilon_{l-1}|] \E[|\varepsilon_{k-1}|] \mathcal{P}_{k+1:n}\mathcal{P}_{l+1:n}}_\text{II},\label{eq:s1234}
		\end{split}
	\end{align}
	where $\mathcal{P}_{i:j} =\prod\limits_{k=i}^{j}(1 - b_kM_{k-1})^2$.	
	In the above, we used strong convexity to bound the first term in \eqref{eq:s123}. Term (II) in \eqref{eq:s1234} is extra when compared to the analysis in the unbiased case. The rest of proof uses the bounds obtained in Lemma \ref{lem:srprime-consistent} to bound terms (I) and (II) on the RHS of \eqref{eq:s1234}.
\end{remark}

\subsubsection{Convex case}
In this subsection, we relax the strong convexity assumption, and work with any convex  $\sr(\cdot)$ function:
\begin{assumption}\label{ass:convexity}
	For any $\theta \in [\theta_l,\theta_u]$, the function $h(\theta) = SR_{\lambda}(\theta)$ satisfies $h''(\theta) \geq 0$.
\end{assumption}
Next, since $[\theta_l,\theta_u]$ is assumed to be a compact and convex subset of $\R$ in the problem \eqref{eq:sr-opt-pb}, it has a finite diameter, as specified in the assumption below. 
\begin{assumption}\label{ass:compactset}
	The set $[\theta_l,\theta_u]$ satisfies $|\theta_1 - \theta_2| \leq D,\ \forall\ \theta_1, \theta_2 \in [\theta_l,\theta_u]$, for some $D > 0$.
\end{assumption}

The stochastic gradient descent expression is given as follows:
\begin{align} \label{eq:sgdconvex}
	\theta_{k+1} = \Pi_{[\theta_l,\theta_u]}(\theta_k - b_k h_m'(\theta_k)),
\end{align}
where $b_k$ is a step-size parameter, $h_m'(\theta_k)$ is an estimate of $\frac{dSR_{\lambda}(\theta)}{d\theta}$ using $m$ samples and $\Pi_{[\theta_l,\theta_u]}$ is the projection on to the set $[\theta_l,\theta_u]$. 

The analysis in the convex case is for an algorithm that requires the knowledge of the horizon $n$, which is the number of iterations for which  \eqref{eq:sgdconvex} is run. Using the value of $n$, we employ the following step-size selection scheme from \cite{jain2019making}:
\begin{align} \label{eq:horizon}
	n_i = n - \lceil 2^{-i}n\rceil,\ 0 \leq i \leq p,\ \text{and}\ n_{p+1} = n,
\end{align}
$p := \text{inf}\{i:2^{-i}n\leq1\}$.
In essence, the above scheme splits the horizon $n$ into $p$ phases, and keeps the step-size constant within a given phase.

\begin{theorem} \label{th:convexsgd}
	Suppose \Crefrange{ass:1}{ass:xiprimeBound}, \ref{ass:convexity} and \ref{ass:compactset} hold. Suppose the update in \eqref{eq:sgdconvex} is performed for $n$ iterations with step-size $b_k$ and batch size $m_k$ set as follows:
	\begin{align}
		b_k = \frac{b_02^{-i}}{\sqrt{n}}, \ \text{and}\ m_k = 2^i n,
		\label{eq:convexstepsize}
	\end{align}
	for some constant $b_0$ when $n_i < k \leq n_{i+1}$, $0 \leq i \leq p$ with $n_i,\ p$ as defined in \eqref{eq:horizon}. Then for any $n \geq 4$,
\begin{align}
			\mathbb{E}[h(\theta_n) - h(\theta^*)] \leq \frac{\mathcal{K}_2}{\sqrt{n}} + \frac{\mathcal{K}_3}{n} + \frac{\mathcal{K}_4}{n^{3/2}},
		\end{align}
		where  $\mathcal{K}_2 = \frac{4D^2}{b_0} + 39DC_4 + 11B_1^2b_0$, $\mathcal{K}_3 = 16b_0B_1C_4$, $\mathcal{K}_4 = \frac{20C_5b_0}{3}$ and $B_1 = \frac{L_1M_2}{\eta}$.
\end{theorem}
\begin{prf}
	See Section \ref{sec:proofs-opt-convex}.
\end{prf}
\begin{remark}
	From the bound above, it is apparent that for obtaining the $O\left(\frac{1}{\sqrt{n}}\right)$ rate, the sample complexity of the algorithm \eqref{eq:sgdconvex} is $n^2\log n$.
\end{remark}

A summary of the iteration and sample complexities for the algorithm \eqref{eq:sr-gd-update} is given in the table below.
\begin{table}[h]
	\centering
	
	\begin{tabular}{|| c | c | c | c ||}
		\hline
		\multirow{2}{*}{Function Type} & \multirow{2}{*}{Bound} & \multirow{2}{*}{Iteration Complexity} & Sample complexity  \\ 
		&&& (with batching) \\
		\hline \hline 
		\vspace{-0.6em} & & &\\ 
		Convex  & $\E\left[h(\theta_n) - h(\theta^*)\right] \leq \epsilon$ &  $n=\mathcal{O}(\frac{1}{\epsilon^2})$
		& $\mathcal{\widetilde O}(\frac{1}{\epsilon^4})$\\ 
		\vspace{-0.6em} & & &\\ \hline
		\vspace{-0.6em} & & &\\ 
		Strongly Convex & $\E[(\theta_{n} - \theta^*)^2] \leq \epsilon$ & $n=\mathcal{O}(\frac{1}{\epsilon})$
		& $\mathcal{O}(\frac{1}{\epsilon^2})$\\ 
		\vspace{0em} & & &\\ \hline
	\end{tabular}\\[1ex]
	\caption{Summary of iteration and sample complexities for UBSR optimization. Here, $\mathcal{\widetilde O}(\cdot)$ is a variant of the big-Oh notation, where the logarthmic factors are ignored.}
	\label{tab:iter_complexity}
\end{table}

From the table above, it is apparent that the sample complexity is $O(1/\epsilon^2)$ for the strongly convex case, while the case of unbiased gradients would lead to a sample complexity of $O(1/\epsilon^2)$. Considering the UBSR derivative estimates are biased, we believe this gap in sample complexity cannot be improved. A similar gap in sample complexity can be seen in \cite{atchade2014stochastic} for a stochastic gradient algorithm with biased gradient information. Likewise, the sample complexity for biased stochastic optmization can be expected to be different from the unbiased variant.

\subsubsection{Comparison to optimization with an inexact gradient oracle}
\label{sec:comparison}
In this section, we compare our contributions in the context of UBSR optimization to previous works that consider stochastic gradient algorithms with inputs from an inexact gradient oracle. A few recent works on this topics are \cite{bhavsar2021nonasymptotic,pasupathy2018sampling,duchi2012finite,karimi2019non,chen2021closing,hu2020biased,hu2021bias,devolder2011stochastic}.
For invoking the results from either \cite{bhavsar2021nonasymptotic} or \cite{pasupathy2018sampling} for UBSR optimization, one requires a non-asymptotic bound for UBSR estimation, which we derive in our paper. In particular, these references consider an abstract optimization setting where the objective function measurements are biased, and the bias can be controlled through a batch size parameter. The bounds in Section \ref{sec:sr-est} would enable UBSR optimization through a stochastic gradient scheme, and the results from these two references would apply.
The gradient estimation scheme in the aforementioned references is based on the idea of simultaneous perturbation (or in simpler terms, finite differences), which is a `black-box' scheme, i.e., does not use the form/structure of the objective function. In contrast, we use the form of the UBSR objective, which in turn leads to an expression for its derivative. Using this expression, we form an estimate of UBSR derivative from i.i.d. samples, and then analyze the statistical properties of the `direct' estimator in Lemma \ref{lem:srprime-consistent}. Thus, the bounds we derive in Theorem \ref{thm:sr-opt-muknown} are specialized to the UBSR optimization problem, leading to precise constants. Finally, in \cite{pasupathy2018sampling}, the authors only provide asymptotic rate results in the form of central limit theorems, while we study the UBSR optimization problem from a non-asymptotic viewpoint. The bounds we derive contain precise guidelines for choosing step-size and batch size parameters, which aid practical implementations.

Next, the stochastic optimization framework considered in \cite{balasubramanian2018zeroth} is not directly applicable for UBSR optimization, as they assume that the objective function measurements have zero-mean noise, while UBSR estimation results in a noise component with a positive mean. The latter can be controlled using the batch size used for estimation.

In \cite{karimi2019non}, the authors derive a non-asymptotic bound of the order $O(\log n/\sqrt{n})$ using a stochastic gradient algorithm for a biased stochastic optimization problem. However, their framework does not feature a batch size parameter, and their result requires the existence of a Lyapunov function.

Another related contribution is \cite{chen2021closing}, a paper that unifies stochastic bilevel, compositional and min-max problems. However, for the UBSR optimization setting we consider, we do not have access to unbiased estimates of the Hessian and cross-Hessians. Therefore the ALSET algorithm framework they propose is not directly applicable for our problem.

In \cite{hu2021bias}, the authors consider a biased stochastic optimization setting. They assume an exponential decay in the bias for a sequence of estimates of the objective as well as its gradient. Such an assumption does not hold for UBSR optimization. In  \cite{hu2020biased}, the authors consider a compositional optimization structure. Although UBSR optmization does not have a compositional structure,  the batch size and the overall sample complexity bound of  $O\left(\frac{1}{\epsilon^2}\right)$ matches the result in \cite{hu2021bias}.

Finally, in \cite{devolder2011stochastic}, the author considers a biased gradient oracle, and provides a $O(1/n)$ bound. However, their results are not directly applicable for UBSR optimization, as they consider a deterministic bias parameter, and their result does not feature a tunable batch-size parameter.

\section{Simulation Experiments}
\label{sec:sr-exp}
	In this section, we describe simulation experiments for the problems of UBSR estimation and optimization\footnote{The implementation is available at  \url{https://github.com/Vishwajit-hegde/Utility-Based-Shortfall-Risk}.}. In particular, we conduct two sets of experiments for UBSR estimation, the first using synthetic data in Section \ref{sec:ubsr-est-synthetic}, and the second using a well-known credit risk model in Section \ref{sec:ubsr-est-creditrisk}. Subsequently,  we describe experiments for UBSR optimization in a portfolio management application in Section \ref{sec:ubsr-opt-portfolio}. 
	
	\subsection{UBSR estimation on synthetic data}
	\label{sec:ubsr-est-synthetic}
	In this setting, we model the underlying distribution using a standard normal random variable. The loss function $l(\cdot)$ used in the definition of UBSR is set as follows: 
	\[l(x) = \eta^{-1}([x]^{+})^{\eta}, \eta>1.\]
	The piece-wise polynomial loss function defined above can be seen to satisfy \Cref{ass:2}. For the experiments, we set $\eta=2$. 
	The parameter $\lambda$, which is used in the definition of the acceptance set \eqref{acc_set}, is set to $0.4$.
	We run the UBSR estimation scheme \eqref{eq:sr-est-update} using two different stepsizes, namely $a_k=\frac{c}{k}$ and $a_k = c/k^{\alpha}$. 
	
	Figure \ref{fig:sr-est-synthetic} presents the estimation error as a function of the number of iterations of \eqref{eq:sr-est-update}, under the two different step size choices namely $a_k=\frac{1}{k}$ and $a_k = 0.1/k^{\alpha}$. For the latter choice, we report results with $\alpha=0.5, 0.7,$ and $0.8$. 
	From Figure \ref{fig:sr-est-synthetic}, it is apparent that the estimation error vanishes under both step size settings. Further, we observe that a larger step size (e.g. $\alpha=0.5$) leads to faster convergence.
	
	\begin{figure}
		\begin{tabular}{cc}
			\begin{subfigure}{0.5\textwidth} 
				\scalebox{0.85}{
					\begin{tikzpicture}
						\begin{axis}[
							xlabel={Iteration $k$},
							ylabel={$(t_k - t^*)^2$},
							legend entries={Dimishing step size $c/k$},
							legend style={legend columns=3, at={(0.9,-0.25)}},
							y label style={at={(0.01,0.5)}}
							]
							\addplot+[mark=none, thick] table [col sep=comma, x index=0, y index=1] {results/sr_estimation_without_alpha.csv};
						\end{axis}
				\end{tikzpicture}}
				\caption{Step size $a_k = 1/k$ and $t_0 = 0.02$.}
				\label{fig:sr-est-1}
			\end{subfigure}
			&
			\begin{subfigure}{0.5\textwidth} 
				\scalebox{0.85}{
					\begin{tikzpicture}
						\begin{axis}[
							xlabel={Iteration $k$},
							ylabel={$(t_k - t^*)^2$},
							legend entries={$\alpha=0.5$, $\alpha = 0.7$, $\alpha = 0.8$},
							legend style={legend columns=3, at={(0.9,-0.25)}},
							y label style={at={(0.01,0.5)}},
							]
							\addplot+[mark=none, thick] table [col sep=comma, x index=0, y index=1] {results/sr_estimation_alpha.csv};
							\addplot+[mark=none, thick] table [col sep=comma, x index=0, y index=2] {results/sr_estimation_alpha.csv};
							\addplot+[mark=none, thick] table [col sep=comma, x index=0, y index=3] {results/sr_estimation_alpha.csv};
						\end{axis}
				\end{tikzpicture}}
				\caption{Step size $a_k = 0.1/k^{\alpha}$ and $t_0 = 0.02$.}
				\label{fig:sr-est-2}
			\end{subfigure}
		\end{tabular}
		\caption{UBSR estimation error for the algorithm \eqref{eq:sr-est-update} using two different step size choices. The results are averages over $10$ independent replications. }
		\label{fig:sr-est-synthetic}
	\end{figure}

	\subsection{UBSR estimation using a credit risk model}
	\label{sec:ubsr-est-creditrisk}
	The online estimation method of UBSR proposed in this paper is applied on the credit risk model studied in \cite{dunkel2010stochastic}. It is also referred to as Normal Copula Model (NCM). Suppose a bank provides capital to $m$ borrowers, and each borrower carries a potential risk of default, resulting in a loss for the bank. The overall loss $L$ is represented  as follows:
	\begin{align*}
		L = \sum_{i=1}^{m} \nu_iD_i,
	\end{align*}
	where $D_i$ is a binary variable with $D_i=1$ corresponding to a default of borrower $i$, and $\nu_i$ is the partial net loss suffered by the bank when $i$ defaults assuming no recovery.\\
	We model $D_i=\indic{R_i>r_i}$, where $r_i=\Phi^{-1}(1-p_i)$, with $p_i$ denoting marginal default probability of $i$. Here $\Phi$ denote the CDF of a standard normal r.v. Further, $R_i$ is determined using the following factor model:
	\begin{align*}
		&R_i = A(i,0)\epsilon_i + \sum_{j=1}^{d}A(i,j)Z_i, i=1,...,m, d<m,\\
		& \textrm{such that }\sum_{j=0}^{d}A(i,j)^2 = 1, A(i,0)>0, A(i,j)\geq 0.
	\end{align*}
	In the above, $Z_1,...,Z_d$ are systematic risk variables, and $\epsilon_1,....,\epsilon_d$ are idiosyncratic risk variables. All the risk variables are chosen as independent standard normal random variables.
	
	For setting the parameters of the credit risk model, we use the same choice as in \cite{dunkel2010stochastic}. In particular, the number of borrowers $m=25$, partial losses $\nu_1=...=\nu_5=1.00$, $\nu_6=...=\nu_{10}=1.25$, $\nu_{11}=...=\nu_{15}=1.50$, $\nu_{16}=\ldots=\nu_{20}=1.75$ and $\nu_{21}=\ldots=\nu_{25}=2.00$. The marginal default probabilities $p_i=0.05, \forall i$. The value of threshold $r_i$ turns out to be $1.64488$. The number of common factors is given by $d=6$. The coupling parameters 
	$A(i,j)$ are given by $A(1,1)=...=A(5,1)=0.1, A(6,2)=...=A(10,2)=0.1, A(11,3)=...=A(15,3)=0.1, A(16,4)=...=A(20,4)=0.1, A(21,5)=...A(25,5)=0.1$, $A(i,6)=0.1$ and $A(i,j)=0$ otherwise for $i=1,...,m$.\\ 
	Loss function for UBSR is chosen to be a piecewise polynomial function $l(x) = \frac1\eta [x]^{+},$ with $\eta=2$. $\lambda$ is set to $0.05$. The value of $\sr(L)$ is $5.11$, see \cite{dunkel2010stochastic}.
	
	\begin{figure}
		\centering
		\tabl{c}{\scalebox{0.9}{
				\begin{tikzpicture}
					\begin{axis}[
						xlabel={number of samples},
						ylabel={$\E[|t_k-\sr|^2]$}, 
						smooth,
						width=8cm,height=7.25cm,
						legend entries={ Online Estimation,  SAA},
						legend pos=north east
						]
						\addplot [thick, red,smooth] table[x index=0,y index=2,col sep=comma,each nth point={15}] {results/plot_data.csv};
						\addplot [thick, blue,smooth] table[x index=0,y index=1,col sep=comma,each nth point={15}] {results/plot_data.csv};
					\end{axis}
			\end{tikzpicture}}\\[1ex]}
		\caption{Comparison of the squared estimation error of UBSR estimate between Sample Average Approximation (SAA) and Online Estimation for credit risk model. The results are averages for $1000$ independent replications.}
		\label{fig:SAAOnlineError}
	\end{figure}
	\begin{table*}
		\centering
		\caption{Comparison of average squared estimation error and total time taken for 1000 replications of the experiment between online estimation and SAA.}
		\label{tab:credit-risk}
		\begin{tabular}{|c|c|c|c|}
			\toprule\hline
			\rowcolor{gray!20}
			\multicolumn{4}{|c|}{\multirow{2}{*}{\textbf{Online Estimation \eqref{eq:sr-est-update}}}}\\[1em]
			\midrule\hline
			\# samples & $100$ & $1000$  & $10000$ \\
			\midrule
			&&&\\
			Squared estimation error & $3.8175$ & $0.6142$ & $0.0838$\\
			\midrule
			&&&\\
			Time taken (in seconds) & \textbf{$0.3762$} & \textbf{$3.0875$} & \textbf{$30.5046$}\\
			\bottomrule\hline
			\rowcolor{gray!20}
			\multicolumn{4}{|c|}{\multirow{2}{*}{\textbf{Sample Average Approximation \cite{zhaolin2016convex}}}}\\[1em]
			\midrule\hline
			\# samples &$100$ & $1000$  & $10000$ \\
			\midrule
			&&&\\
			Squared estimation error & $0.8488$ & $0.1517$ & $0.0539$\\
			\midrule
			&&&\\
			Time taken (in seconds) & $1.8109$ & $9.1270$ & $145.4006$\\\hline
			\bottomrule 
		\end{tabular}
	\end{table*}
	
	We perform UBSR estimation using our algorithm \eqref{eq:sr-est-update} and compare its performance to the Sample Average Approximation (SAA) based UBSR estimation scheme in \cite{zhaolin2016convex}.  
	Figure \ref{fig:SAAOnlineError} presents the mean-squared error for the two UBSR estimation algorithms mentioned above. The results are averages over  $1000$ independent replications. Table \ref{tab:credit-risk} tabulates the mean-squared and the time taken for both the algorithms for $100$, $1000$ and $10000$ samples.
	In order to apply SAA in the online setting, it is assumed that UBSR is recalculated after every $10$ samples in the case of SAA.
	From Table \ref{tab:credit-risk}, the time taken for our online estimation scheme is almost a factor of five smaller than SAA for $10000$ samples, for a comparable mean-squared error.

	\begin{figure}
		\begin{tabular}{cc}
			\begin{subfigure}{0.5\textwidth} 
				\scalebox{0.85}{
					\begin{tikzpicture}
						\begin{axis}[
							xlabel={Iteration $k$},
							ylabel={$\|\theta_k - \theta^*\|^2$},
							legend entries={Dimishing step size $b_k=c/k$, Geometric step size \eqref{eq:convexstepsize}},
							legend style={legend columns=1, at={(0.9,-0.25)}},
							y label style={at={(0.01,0.5)}}
							]
							\addplot+[mark=none, thick] table [col sep=comma, x index=0, y index=1] {results/markowitz_theta_error_squared_strongly_convex_case.csv};
							\addplot+[mark=none, thick] table [col sep=comma, x index=0, y index=1] {results/markowitz_theta_error_squared_convex_case.csv};
						\end{axis}
				\end{tikzpicture}}
				\caption{Error in parameter value}
				\label{fig:ubsr-param-diff}
			\end{subfigure}
			&
			\begin{subfigure}{0.5\textwidth} 
				\scalebox{0.85}{
					\begin{tikzpicture}
						\begin{axis}[
							xlabel={Iteration $k$},
							ylabel={$|\sr(\theta_{k}) - \sr(\theta^*)|$},
							legend entries={Dimishing step size $b_k=c/k$, Geometric step size \eqref{eq:convexstepsize}},
							legend style={legend columns=1, at={(0.9,-0.25)}},
							y label style={at={(-0.05,0.5)}},
							]
							\addplot+[mark=none, thick] table [col sep=comma, x index=0, y index=1] {results/markowitz_ubsr_error_strongly_convex_case.csv};
							\addplot+[mark=none, thick] table [col sep=comma, x index=0, y index=1] {results/markowitz_ubsr_error_convex_case.csv};
						\end{axis}
				\end{tikzpicture}}
				\caption{Error in UBSR value}
				\label{fig:ubsr-func-diff}
			\end{subfigure}
		\end{tabular}
		\caption{Performance evaluation of our UBSR optimization algorithm \eqref{eq:sgdconvex} on a portfolio management problem. The results are averages over $10$ independent replications. }
		\label{fig:ubsr-opt-results}
	\end{figure}
	
	\subsection{UBSR optimization for portfolio management}
	\label{sec:ubsr-opt-portfolio}
	We consider the portfolio management experiment described in \cite[Section 5.2]{zhaolin2016convex} to illustrate the application of our UBSR optimization algorithm \eqref{eq:sr-gd-update}. 
	In this setting, we have a portfolio composed of $d$ assets with return $\mathbf{r}=(r_1,\ldots,r_d)$. The distribution of $\mathbf{r}$ is assumed to be a multivariate Gaussian with mean $\mu$ and covariance $\Sigma$. We consider the portfolio optimization problem in the \textit{Markowitz} sense, i.e., we aim is to find an allocation $\theta=(\theta_1,\ldots,\theta_d)$ that optimizes the UBSR of the financial position $X=\br\tr\theta$, while guaranteeing a minimum expected return $R_0$. More precisely, the constrained optimization problem can be written as
	\begin{equation}
		\begin{aligned}
			\min_{\theta} \sr(\mathbf{r}\tr\theta) \quad\textrm{subject to}\\
			\sum_{i=1}^d \theta_i=1,\ \theta_i\ge 0, i=1,\ldots, d, \\
			\mu\tr\theta \ge R_0.
		\end{aligned}
		\label{eq:sr-markowitz}
	\end{equation}	
	For a given $\theta$, the UBSR can be expressed in closed form, with the loss function  $\ell(-\theta) = \exp(-\beta \theta)$ as follows:
	\begin{align*}
		\sr(\mathbf{r}\tr\theta) &= - \mu\tr\theta + \frac{\beta}{2}\theta\tr\Sigma \theta - \frac{\log \lambda}{\beta}, 
	\end{align*}
	For the constrained problem \eqref{eq:sr-markowitz}, the Lagrangian\footnote{We ignore the $\theta_i\ge 0$ constraint in writing the Lagrangian, as the choice of $R_0$ that we employed in our experiments ensured that these constraints are inactive, forcing the corresponding Lagrange multipliers to be zero due to complementary slackness.} is given by 
	\begin{align}
		L(\theta,\vartheta_1,\vartheta_2)=\sr(\mathbf{r}\tr\theta) + \vartheta_1 \left(\sum_{i=1}^d \theta_i-1\right) + \vartheta_2\left( \mu\tr\theta - R_0\right).
	\end{align}
	Using the KKT conditions, the optimal Lagrange multipliers can be derived as
	\begin{align}
		\vartheta_1^* = \beta \frac{(\mu\tr\Sigma^{-1}\mathbf{1})-R_0(\mathbf{1}\tr\Sigma^{-1}\mathbf{1})}{(\mathbf{1}\tr\Sigma^{-1}\mu)(\mu\tr\Sigma^{-1}\mathbf{1}) - (\mu\tr\Sigma^{-1}\mu)(\mathbf{1}\tr\Sigma^{-1}\mathbf{1})} - 1,\textrm{ and }\\
		\vartheta_2^* = \beta \frac{R_0(\mathbf{1}\tr\Sigma^{-1}\mu)-\mu\tr\Sigma^{-1}\mu}{(\mathbf{1}\tr\Sigma^{-1}\mu)(\mu\tr\Sigma^{-1}\mathbf{1}) - (\mu\tr\Sigma^{-1}\mu)(\mathbf{1}\tr\Sigma^{-1}\mathbf{1})},
	\end{align}	
	where $\mathbf{1}$ is the column vector of all ones.
	The optimal allocation  $ \theta^*$ is given by   
	\begin{align*}
		\theta^* = \frac1\beta ((1+\lambda_1)\Sigma^{-1}\mu + \lambda_2 \Sigma^{-1}\mathbf{1}).
	\end{align*}
	
	In our experiments, we set $d=3$, 
	$\mu=\begin{bmatrix} 0.13 &0.1&  0.08\end{bmatrix}$,
	$\Sigma=\begin{bmatrix} 
		0.05  &  0.004 &  0.0002\\
		0.004  & 0.01  & -0.0005\\
		0.0002 &-0.0005 & 0.001\end{bmatrix}$, $R_0=0.1$, $\lambda=0.1$ and $\beta = 5$. Since $\Sigma$ is positive definite, the UBSR function is strongly convex, making the bounds in Theorem \ref{thm:sr-opt-muknown} as well as Theorem \ref{th:convexsgd} applicable.
	
	We run our UBSR optimization algorithms for the strongly convex and convex cases, respectively, with a constant batch size $m=100$. The strongly convex case uses a step size of the form $c/(c+k)$, with $c$ chosen to satisfy the constraint $\mu_2 c> \frac{1}{2}$, while the step size in the convex case uses the form given in \eqref{eq:convexstepsize} with the constant $a_0$ set to $50$. The diminishing step size choice used in the implementation is slightly different from the $c/k$ form used in Theorem \ref{thm:sr-opt-muknown}. However, the aforementioned theorem can be easily extended to cover a step size of the form $c/(c+k)$. The matrix $\Sigma$ chosen for the experiments has a very low $\mu_2$ (minimum eigenvalue), motivating the step size form  $c/(c+k)$, instead of $c/k$, since $c$ is inversely proportional to $\mu_2$.
	
	For the UBSR optimization algorithms, with updates governed by \eqref{eq:sr-gd-update} and \eqref{eq:sgdconvex}, respectively, the initial point $\theta_{0}$ was set to $\begin{bmatrix} 1 &0&  0\end{bmatrix}$.  The number of iterations $n$ is $500$. Figure \ref{fig:ubsr-opt-results} presents the results obtained for both algorithms. In particular, Figure \ref{fig:ubsr-param-diff} presents the error in the parameter value, which is computed using $\theta^*$ defined above, while Figure \ref{fig:ubsr-func-diff} presents the error in the UBSR value. 
	
	From Figure \ref{fig:ubsr-param-diff}, it is apparent that both algorithms converge, implying that both diminishing step sizes as well as geometric step size work in practice. Among these choices for the step sizes, the geometric step size slightly outperforms the diminishing variant.. 
	We observe a similar convergence behaviour from the results in Figure \ref{fig:ubsr-func-diff}, implying that our algorithms converge both in the parameter as well as in UBSR value.

\section{Proofs for SR estimation}
\label{sec:proofs-est}
\subsection{Proof of Theorem \ref{thm:sr-est-muknown}}
\label{appendix:proof-sr-est1}
\begin{prf1}
	From the update rule \eqref{eq:sr-est-update}, and the fact that $\sr(X)$ lies within the projected region $[t_l,t_u]$, we obtain
	\begin{align}
		z_{n} &= t_{n} - \sr(X)= \mathcal{T}(t_{n-1} +  a_{n}\left ( g(t_{n-1}) +  \varepsilon_{n-1}\right)) - \mathcal{T}(\sr(X))\nonumber\\
		&= \mathcal{T}(t_{n-1} +  a_{n}\left ( g(t_{n-1}) +  \varepsilon_{n-1}\right)) - \sr(X). \label{eq:p1}
	\end{align}
	For any $k\ge 1$, define
	\begin{equation}
		J_{k}  =  \int\limits_{0}^{1} g'(m t_{k} + (1-m)\sr(X))dm.
		\label{eq:Jn}
	\end{equation}
	
	Using \Cref{ass:2}, we obtain $J_{k} \leq -\mu_1,\ $for all $k\ge 1$. Using $J_n$ we can express $g(t_{n})$ as,
	\[
	g(t_{n}) = \int\limits_{0}^{1} g'(m t_{n} + (1-m)\sr(X))dm (t_n-\sr(X)) = J_n z_n.
	\]
	
	Squaring on both sides of \eqref{eq:p1}, and using the fact that projection is non-expansive (see Lemma 10 in \cite{nithia2021}), we obtain
	\begin{align*}
		z_{n}^{2} &\leq [z_{n-1} + a_{n}(g(t_{n-1}) + \varepsilon_{n-1})]^{2}\\
		&\leq [z_{n-1} + a_{n}(J_{n-1}z_{n-1} + \varepsilon_{n-1})]^{2}\\
		&\leq [z_{n-1}(1 + a_{n}J_{n-1}) + a_{n}\varepsilon_{n-1}]^{2}\\
		&\leq z_{n-1}^{2}(1 + a_{n}J_{n-1})^{2} + a_{n}^{2}\varepsilon_{n-1}^{2} + 2z_{n-1}(1 + a_{n}J_{n-1})a_{n}\varepsilon_{n-1}.
	\end{align*}
		Taking conditional expectation given $\mathcal{F}_{n-1}$, where $\mathcal{F}_{n-1}$ is the $\sigma$-field generated by $\{t_k, k < n\}$, and using $\E[\varepsilon_{n-1}|\mathcal{F}_{n-1}] = 0$, we obtain (a.s.)
		\begin{align*}
			\E[z_{n}^{2}|\mathcal{F}_{n-1}] &\leq (1 + a_{n}J_{n-1})^{2} z_{n-1}^{2} + a_{n}^{2}\E[\varepsilon_{n-1}^{2}|\mathcal{F}_{n-1}] +2 z_{n-1}(1 + a_{n}J_{n-1})a_{n}\E[\varepsilon_{n-1}|\mathcal{F}_{n-1}]\\
			&\leq (1 + a_{n}J_{n-1})^{2}z_{n-1}^{2} + a_{n}^{2}\E[\varepsilon_{n-1}^{2}|\mathcal{F}_{n-1}].\numberthis\label{eq:se1}
	\end{align*}
		From \Cref{ass:2}, we know that  $-L_1 \leq J_k \leq -\mu_1$ for all $k\in\{1,\ldots,n-1\}$. Using this fact, we have
		\begin{align}
			(1 + a_{k}J_{k-1})^{2} = 1 + 2a_kJ_{k-1} + a_k^2J_{k-1}^2 \leq 1 - 2a_k\mu_1 + a_k^2L_1^2 \leq \exp\left(-2\mu_1a_k + L_1^2a_k^2\right).\label{eq:jkbound}
		\end{align}
		Using \eqref{eq:jkbound} in \eqref{eq:se1} and taking total expectations, we obtain
		\begin{align}
			\E[z_{n}^{2}] &\leq \exp\left(-2\mu_1a_n + L_1^2a_n^2\right) \E[z_{n-1}^{2}] + a_{n}^{2}\sigma^{2} \nonumber\\
			&\leq z_{0}^{2}\prod\limits_{k=1}^{n}(1 + a_{k}J_{k-1})^{2} + \sigma^{2}\sum\limits_{k=1}^{n}[a_{k}^{2}\prod\limits_{j=k+1}^{n}(1 + a_{j}J_{j-1})^{2}].\label{eq:expec-bd}
	\end{align}
	Hence, we obtain
	\begin{align*}
		\E[z_{n}^{2}] &\leq  z_{0}^{2}\exp\left(-2\mu_1\sum\limits_{k=1}^{n}a_{k}+L_1^2\sum\limits_{k=1}^{n}a_k^2\right) + \\ & \qquad \sigma^{2}\sum\limits_{k=1}^{n}a_{k}^{2}\exp\left(-2\mu_1\sum\limits_{j=k+1}^{n}a_{j} + L_1^2\sum\limits_{j=k+1}^{n}a_j^2\right)\\
		&\leq z_{0}^{2}\exp\left(-2\mu_1 c \log n  + \frac{L_1^2c^2\pi^2}{6}\right) + \\ & \qquad \sigma^{2}\sum\limits_{k=1}^{n}a_{k}^{2}\exp\left(-2\mu_1 c \log\left(\frac{n}{k+1}\right) +\frac{L_1^2c^2\pi^2}{6}\right)\numberthis\label{eq:pisquare6}\\
		&\leq \exp\left(\frac{L_1^2c^2\pi^2}{6}\right)\left[\frac{z_{0}^{2}}{n^{2\mu_1c}} + \sigma^{2}\sum\limits_{k=1}^{n}a_{k}^{2}\left(\frac{n}{k+1}\right)^{-2\mu_1 c}\right]\\
		&\leq \exp\left(\frac{L_1^2c^2\pi^2}{6}\right)\left[\frac{z_{0}^{2}}{n^{2\mu_1 c}} + \sigma^{2}n^{-2\mu_1 c}\sum\limits_{k=1}^{n}\frac{c^{2}}{k^{2}}(k+1)^{2\mu_1 c}\right]\\
		&\leq \exp\left(\frac{L_1^2c^2\pi^2}{6}\right)\left[\frac{z_{0}^{2}}{n^{2\mu_1 c}} + \sigma^{2}\left(\frac{2}{n}\right)^{2\mu_1 c}\sum\limits_{k=1}^{n}c^{2}k^{2\mu_1 c - 2}\right]\stepcounter{equation}\tag{\theequation}\label{eq:p2}\\[1ex]
		&\leq \exp\left(\frac{L_1^2c^2\pi^2}{6}\right)\left[\frac{z_{0}^{2}}{n^{2\mu_1 c}} + \sigma^{2}2^{4\mu_1 c}\frac{c^{2}}{(2\mu_1 c - 1)}\frac{1}{n}\right].
	\end{align*}
	For the inequality in \eqref{eq:pisquare6}, we have used 
	$\sum_{k=1}^{n}\frac{1}{k^2} \leq \sum_{k=1}^{\infty}\frac{1}{k^2} = \frac{\pi^2}{6}.$
	For bounding the sum in \eqref{eq:p2}, we used
	\begin{align}
		\frac{1}{n^{2\mu_1 c}} \sum\limits_{k=1}^{n}k^{2\mu_1 c - 2} &\leq  \frac{1}{n^{2\mu_1 c}}\int\limits_{0}^{n+1}k^{2\mu_1 c - 2}dk
		\leq \frac{(n+1)^{2\mu_1 c - 1}}{n^{2\mu_1 c}(2\mu_1 c - 1)} \leq \frac{2^{\mu_1 c}}{(2\mu_1 c - 1)}\frac{1}{n}.\label{eq:sum-bound}
	\end{align}
	
	Hence proved.
\end{prf1}
\subsection{Proof of Theorem \ref{thm:sr-est-muknown-hpb}}
\label{appendix:proof-sr-est1-hpb}
\begin{prf1}
	We use the technique from \cite{frikha2012concentration}, and tailor the analysis to the SR estimation problem, instead of a general stochastic approximation scheme in \cite{frikha2012concentration}. Moreoever, unlike the bounds in the aforementioned reference, we make all the constants explicit.
	
	The centered form of the iterate $z_n = t_n - \sr(X)$ can be written as a telescoping sum as follows:
	\[
	|z_n| - \E[|z_n|] = \sum\limits_{k=1}^n g_k - g_{k-1} = \sum\limits_{k=1}^n D_k,
	\]
	where $g_k = \E[|z_k| |\mathcal{F}_k], D_k = g_k - g_{k-1}$ and $\mathcal{F}_k = \sigma(t_1,\ldots,t_k)$.\\
	Let $t_j^i(t)$ denote the iterate at time instant  $j$, given that $t_i = t$. Using this notation, we have
	\begin{align*}
		& \E[|t_{j+1}^{i}(t) - t_{j+1}^{i}(t')|^2] \\
		&\leq \E[|t_{j}^{i}(t) - t_{j}^{i}(t') + a_j(\hat{g}(t_{j+1}^{i}(t)) - \hat{g}(t_{j}^{i}(t'))|^2] \\
		&\leq \E[|t_{j}^{i}(t) - t_{j}^{i}(t')|^2] + 2a_j \E[t_{j+1}^{i}(t) - t_{j+1}^{i}(t')]\E[\hat{g}(t_{j+1}^{i}(t)) - \hat{g}(t_{j}^{i}(t'))] + \\ & \qquad a_J^2 \E[|\hat{g}(t_{j+1}^{i}(t)) - \hat{g}(t_{j}^{i}(t'))|^2]\\
		&\leq (1 - 2\mu_1a_j + a_j^2 L_1^2) \E[|t_{j}^{i}(t) - t_{j}^{i}(t')|^2].
	\end{align*}
	Unrolling the recursion above, we obtain 
	\[
	\E[|t_{n}^{i}(t) - t_{n}^{i}(t')|^2] \leq |t-t'|^2 \prod_{j=1}^n (1-2\mu_1a_j + a_j^2L_1^2),
	\]
	leading  to
	\begin{align*}
		\E[|t_n - \sr(X)| |t_i =t]  -  \E[|t_n - \sr(X)| |t_i =t'] &\leq \E[|t_n^i(t) - t_n^i(t')|]\\
		&\leq |t-t'| (\prod_{j=1}^{n-1}(1-2\mu_1a_j + a_j^2L_1^2) )^{1/2}\\
		&\leq a_i |\hat{g} - \hat{g}'|(\prod_{j=1}^{n-1}(1-2\mu_1a_j + a_j^2L_1^2) )^{1/2}\\
		&\leq \Gamma_i |\hat{g} - \hat{g}'|.
	\end{align*}
	where $\Gamma_i = a_i(\prod_{j=i}^{n-1}(1-2\mu_1a_j + a_j^2L_1^2) )^{1/2} $, $t = t_{i-1} + a_i\hat{g},$ and $t' = t_{i-1} + a_i\hat{g}'.$
	Now,
	\begin{align*}
		\mathbb{P}(|z_n| - \E[|z_n|] >\varepsilon)  &= \mathbb{}P(\sum\limits_{k=1}^n D_k > \varepsilon)\\
		&\leq \exp(-\lambda \varepsilon)(\E[\exp(\lambda\sum\limits_{k=1}^n D_k )])\\
		&\leq \exp(-\lambda \varepsilon)\E[\exp(\lambda\sum\limits_{k=1}^{n-1} D_k )]\E[\exp(\lambda D_n )|\mathcal{F}_{n-1}].\stepcounter{equation}\tag{\theequation}\label{eq:n1}
	\end{align*}
	From the proof passage in  \cite[p. 585]{prashanth2020concentration}, it can be seen that an $\Gamma$-Lipschitz function $f$ of a r.v. $Z$ satisfying $|Z|\le \B$ is $\Gamma^2 \B^2$-sub-Gaussian, i.e.,
	\begin{align*}
		\E[\exp(\lambda(f(Z))] &\leq \exp\left(\frac{\lambda^2  \Gamma^2 \B^2}{2}\right).
	\end{align*}
	Using \Cref{ass:Xbdd}, and the fact that $\ell$ is $L_1$ Lipschitz, we have $\hat g$ is $L_1^2 \B^2$-sub-Gaussian. Next, $D_n$ is a $\Gamma_n$-Lipschitz function of $\hat g$, implying $D_n$ is $4\Gamma_n^2 L_1^2 \B^2$-sub-Gaussian.
	Using the bound above in \eqref{eq:n1}, we obtain
	\begin{align*}
		\E[\exp(\lambda D_n )|\mathcal{F}_{n-1}] &\leq \exp\left(2\lambda^2 \Gamma_n^2 L_1^2 \B^2\right).
	\end{align*}
	Plugging this bound into \eqref{eq:n1}, followed by an optimization over  $\lambda$, we obtain
	\begin{align}
		\mathbb{P}(|z_n| - \E[|z_n|] >\varepsilon)  &\leq \exp(-\lambda \varepsilon)\exp(2\lambda^2L_1^2\B^2\sum\limits_{k=1}^{n}\Gamma_k^2)
		\leq \exp\left(-\frac{\varepsilon^2}{16 L_1^2\B^2\sum\limits_{k=1}^{n}\Gamma_k^2}\right). \label{eq:hpb1}
	\end{align}
	We now specialize the bound in \eqref{eq:hpb1} using $a_k = c/k$, with $1/2< \mu_1 c $.
	In particular, we first compute $\sum\limits_{k=1}^{n}\Gamma_k^2$ for this stepsize choice, and subsequently derive the high probability bound. 
	\begin{align*}
		\sum\limits_{k=1}^{n}\Gamma_k^2 &= \sum\limits_{k=1}^{n} a_k^2\left(\prod_{j=k}^{n-1}(1-2\mu_1a_j + a_j^2L_1^2)\right) \leq \sum\limits_{k=1}^{n} a_k^2  \exp\left(-2\mu_1\sum_{j=k}^{n}a_j\right)\exp\left(L_1^2\sum_{j=k}^{n}a_j^2\right) \\
		&\leq \exp\left(\frac{L_1^2c^2\pi^2}{6}\right)\sum\limits_{k=1}^{n} a_k^2\exp\left(-2\mu_1\sum_{j=k}^{n}a_j\right) \\ &  \leq \exp\left(\frac{L_1^2c^2\pi^2}{6}\right)\sum\limits_{k=1}^{n} a_k^2 \exp\left(-2\mu_1c\log\left(\frac{n}{k+1}\right)\right) \\
		&\leq  \exp\left(\frac{L_1^2c^2\pi^2}{6}\right)\sum\limits_{k=1}^{n} \frac{c^2}{k^2}\left(\frac{k+1}{n}\right)^{2\mu_1c} \leq \exp\left(\frac{L_1^2c^2\pi^2}{6}\right)\frac{2^{4\mu_1 c}c^{2}}{(2\mu_1 c - 1)}\frac{1}{n}.
	\end{align*}
	Using the bound on $\sum\limits_{k=1}^{n}\Gamma_k^2$ in \eqref{eq:hpb1}, we obtain
	\begin{align}
		\P(|z_n| - \E[|z_n|] >\varepsilon)  &\leq \exp\left(- C_1 n\varepsilon^2\right), \label{eq:hpb21}
	\end{align}
	where $C_1 = \dfrac{(2\mu_1 c - 1)}{2^{4\mu_1 c + 4}c^{2} L_1^2\B^2}\exp\left(-\frac{L_1^2c^2\pi^2}{6}\right)$.
	Using the bound on $\E[|z_n|]$ from Theorem \ref{thm:sr-est-muknown} in \eqref{eq:hpb21}, we have
	\begin{align*}
		&\P\left(|z_n| - E|z_n| \le \sqrt{\frac{\log\left(1/\delta\right)}{\tilde c n}} + \exp\left(\frac{L_1^2c^2\pi^2}{12}\right)\left[
		\frac{\E[|t_1-\sr(X)|]}{n^{\mu_1 c}} + \frac{c \sigma 2^{2\mu_1 c}}{\sqrt{(2\mu_1 c - 1)}\sqrt{n}}\right]
		\right)& \\& \ge 1-\delta.
	\end{align*}
	Hence proved.
\end{prf1}

\subsection{Proof of Theorem \ref{thm:sr-est-muuniv}}
\label{appendix:proof-sr-est2}
\begin{prf1}
	The passage leading up to \eqref{eq:expec-bd} holds for any choice of stepsize, and does not require the knowledge of $\mu_1$ for setting the stepsize constant $c$. Using \eqref{eq:expec-bd} as the starting point, we have
	\begin{align}
		\E[z_{n}^{2}] &\leq z_{0}^{2}\prod\limits_{k=1}^{n}(1 + a_{k}J_{k-1})^{2} + \sigma^{2}\sum\limits_{k=1}^{n}[a_{k}^{2}\prod\limits_{j=k+1}^{n}(1 + a_{j}J_{j-1})^{2}].\label{eq:p3}
	\end{align}
	We split the terms on the RHS above into two regimes: $k<n_0$ and $k\geq n_0$. From \Cref{ass:2}, we have $|J_k|<L_1$. We shall now simplify \eqref{eq:p3} under two different stepsize choices.\\[1ex]
	\textbf{Case I:} $a_k = \frac{c}{k}$\\[0.5ex]
	Notice that
	\begin{align*}
		\prod\limits_{k=1}^{n}(1 + a_{k}J_{k-1})^{2} &= \prod\limits_{k=1}^{n_0}(1 + a_{k}^2J_{k-1}^2 + 2a_kJ_{k-1})\prod\limits_{k=n_0+1}^{n}(1 + a_{k}J_{k-1})^{2}\\ 
		&\leq \prod\limits_{k=1}^{n_0}(1 + a_{k}^2L_1^2 + 2a_kL_1)\prod\limits_{k=n_0+1}^{n}(1 + a_{k}J_{k-1})^{2}\\
		&\leq (1 + cL_1)^{2n_0} e^{-\mu_1 \sum\limits_{n_0+1}^{n} a_k}\\
		&\leq (1 + cL_1)^{2n_0} e^{-\mu_1 c \log\left(\frac{n}{n_0+1}\right)}\leq (1 + cL_1)^{2n_0}  \left(\frac{n_0+1}{n}\right)^{\mu_1 c}\\
		&\leq C(n_0) \frac{1}{n^{\mu_1 c}},
	\end{align*}
	where $C(n_0) = (1 + cL_1)^{2n_0} (n_0+1)^{\mu_1 c}$.
	
	We now handle the second term in \eqref{eq:p3} as follows:
	\begin{align}
		\sum\limits_{k=1}^{n}[a_{k}^{2}\prod\limits_{j=k+1}^{n}(1 + a_{j}J_{j-1})^{2}] &= \sum\limits_{k=1}^{n_0-1}[a_{k}^{2}\prod\limits_{j=k+1}^{n}(1 + a_{j}J_{j-1})^{2}]+\sum\limits_{k=n_0}^{n}[a_{k}^{2}\prod\limits_{j=k+1}^{n}(1 + a_{j}J_{j-1})^{2}]\nonumber\\
		&\leq (1 + cL_1)^{2n_0} \left(\frac{n_0+1}{n}\right)^{\mu_1 c}\,\,\sum\limits_{k=1}^{n_0-1}a_{k}^{2}+\sum\limits_{k=n_0}^{n}a_{k}^{2}\left(\frac{k+1}{n}\right)^{\mu_1 c}\nonumber\\
		&\leq (1 + cL_1)^{2n_0}  (n_0+1)^{\mu_1 c}\frac{\pi^2}{6} \frac{1}{n^{\mu_1 c}}+\frac{c^2}{n^{\mu_1 c}}\sum\limits_{k=n_0}^{n}\frac{c^2}{k^2}(k+1)^{\mu_1 c}. \label{eq:mu-gen}
	\end{align}
	In the above, we used $\sum\limits_{k=1}^{n}a_k^{2} = \sum\limits_{k=1}^{n}\frac{c^2}{k^2} < c^2 \frac{\pi^2}{6}$ to arrive at the inequality in \eqref{eq:mu-gen}.
	
	We now simplify \eqref{eq:mu-gen} based on the value of $\mu_1c$ in the following three cases:\\[1ex]
	\textbf{Case a:} $\mu_1c>1$\\
	Using the bound in \eqref{eq:sum-bound}, we have $ \sum\limits_{k=n_0}^{n}\frac{c^2}{k^2}(\frac{k+1}{n})^{\mu_1 c} \leq \frac{2^{\mu_1 c}c^{2}}{(\mu_1 c - 1)}\frac{1}{n}$. Substituting this bound in \eqref{eq:mu-gen}, we obtain
	\begin{equation}
		\E[z_{n}^{2}] \leq C(n_0)\left(z_{0}^{2} + \sigma^{2}\frac{\pi^2}{6}\right)\frac{1}{{n}^{\mu_1 c}} +\frac{\sigma^{2}c^2 2^{\mu_1 c}}{(\mu_1 c -1)}\frac{1}{n}.
	\end{equation}
	\textbf{Case b:} $\mu_1c=1$\\
	In this case, we have 
	\[ \sum\limits_{k=n_0}^{n}\frac{c^2}{k^2}\left(\frac{k+1}{n}\right)^{\mu_1 c} \leq \frac{2}{n}\sum\limits_{k=n_0}^{n}\frac{c^2}{k}\leq 2c^{2}\frac{\log(n+1)}{n}.\] 
	Substituting the bound derived above in \eqref{eq:mu-gen}, we obtain
	\begin{equation}
		\E[z_{n}^{2}] \leq C(n_0)\left(z_{0}^{2} + \sigma^{2}\frac{\pi^2}{6}\right)\frac{1}{{n}} +2\sigma^{2}c^2 \frac{\log(n+1)}{n}.
	\end{equation}
	\textbf{Case c:} $\mu_1c<1$\\
	In this case, we can infer that
	\[ \frac{1}{n^{\mu_1c}}\sum\limits_{k=n_0}^{n}\frac{c^2}{k^2}(k+1)^{\mu_1 c} \leq \frac{2^{\mu_1c}}{n^{\mu_1c}}\sum\limits_{k=n_0}^{n}\frac{c^2}{k^{(1+(1-\mu_1c))}}\leq \frac{2^{\mu_1c}c^2}{(1-\mu_1c)n},\] 
	leading to the following overall bound:
	\begin{equation}
		\E[z_{n}^{2}] \leq C(n_0)\left(z_{0}^{2} + \sigma^{2}\frac{\pi^2}{6}\right)\frac{1}{{n}^{\mu_1 c}} +\sigma^{2}\frac{2^{\mu_1c}c^2}{(1-\mu_1c)}\frac{1}{n}.
	\end{equation}
	\ \\
	We now turn to analyzing the case when the stepsize $a_k$ is larger than $c/k$.\\[1ex]
	\textbf{Case II:} $a_k = \frac{c}{k^{\alpha}}$ for $\alpha \in (0,1)$.\\[1ex]
	First, we bound a factor in the first term of \eqref{eq:p3} as follows:
	\begin{align*}
		\prod\limits_{k=1}^{n}(1 + a_{k}J_{k-1})^{2} &= \prod\limits_{k=1}^{n_0}(1 + a_{k}^2J_{k-1}^2 + 2a_kJ_k)\prod\limits_{k=n_0+1}^{n}(1 + a_{k}J_{k-1})^{2}\\
		&\leq (1 + cL_1)^{2n_0} \exp\left(-\mu_1 \sum\limits_{n_0+1}^{n} a_k\right)\\
		&\leq (1 + cL_1)^{2n_0}  \exp\left(-\frac{\mu_1 c(n^{1-\alpha}-n_0^{1-\alpha})}{1-\alpha}\right)\\
		&\leq (1 + cL_1)^{2n_0}   \exp\left(\frac{\mu_1 cn_0^{1-\alpha}}{1-\alpha}\right) \exp\left(-\frac{\mu_1 cn^{1-\alpha}}{1-\alpha}\right)\\
		&\leq \widetilde C(n_0) \exp\left(-\frac{\mu_1 cn^{1-\alpha}}{1-\alpha}\right),\stepcounter{equation}\tag{\theequation}\label{eq:p5}
	\end{align*}
	where $\widetilde C(n_0) = (1 + cL_1)^{2n_0}  \exp\left(\frac{\mu_1 cn_0^{1-\alpha}}{1-\alpha} \right)$.
	
	We now bound the second term in \eqref{eq:p3} by splitting the term around $n_0$ as follows:
	\begin{align*}
		&\sum\limits_{k=1}^{n}[a_{k}^{2}\prod\limits_{j=k+1}^{n}(1 + a_{j}J_{j-1})^{2}] \\
		&= \sum\limits_{k=1}^{n_0-1}[a_{k}^{2}\prod\limits_{j=k+1}^{n}(1 + a_{j}J_{j-1})^{2}]+\sum\limits_{k=n_0}^{n}[a_{k}^{2}\prod\limits_{j=k+1}^{n}(1 + a_{j}J_{j-1})^{2}]\\
		&\leq \widetilde C(n_0) \exp\left(-\frac{\mu_1 cn^{1-\alpha}}{1-\alpha}\right)\sum\limits_{k=1}^{n_0-1}a_{k}^{2} +\sum\limits_{k=n_0}^{n}a_{k}^{2}\exp\left(-\frac{\mu_1 c(n^{1-\alpha}-k^{1-\alpha})}{1-\alpha}\right)\\
		&\leq \widetilde C(n_0)c^2 n_0 \exp\left(-\frac{\mu_1 cn^{1-\alpha}}{1-\alpha}\right)+c^2\exp\left(-\frac{\mu_1 cn^{1-\alpha}}{1-\alpha}\right) \sum\limits_{k=n_0}^{n}k^{-2\alpha}\exp\left(\frac{\mu_1 ck^{1-\alpha}}{1-\alpha} \right)\stepcounter{equation}\tag{\theequation}\label{eq:p65}\\
		&\leq \widetilde C(n_0)c^2 n_0 \exp\left(-\frac{\mu_1 cn^{1-\alpha}}{1-\alpha}\right)+\frac{2(\mu_1 c)^{\frac{\alpha}{1-\alpha}}c^2}{1-\alpha} \frac{1}{n^{\alpha}}\stepcounter{equation}\tag{\theequation}\label{eq:p6}.
	\end{align*}
	In arriving at \eqref{eq:p6}, we have bounded the sum $c^2\exp(-\frac{\mu_1 cn^{1-\alpha}}{1-\alpha}) \sum\limits_{k=n_0}^{n}k^{-2\alpha}\exp(\frac{\mu_1 ck^{1-\alpha}}{1-\alpha} )$ in \eqref{eq:p65} by using arguments similar to those used in arriving at  \cite[Eq. (79)]{prashanth2020concentration}. In particular, the latter bound uses Jensen's inequality and the convexity of $f(x) = x^{-2\alpha}\exp(x^{1-\alpha})$. 
	
	Substituting the bounds in \eqref{eq:p5} and \eqref{eq:p6} in \eqref{eq:p3}, we obtain
	\begin{equation}
		\E[z_{n}^{2}] \leq \widetilde C(n_0)\left(z_{0}^{2} + \sigma^{2}c^2 n_0\right)\exp\left(-\frac{\mu_1 cn^{1-\alpha}}{1-\alpha}
		\right) +\frac{\sigma^{2}2(\mu_1 c)^{\frac{\alpha}{1-\alpha}}c^2}{(1-\alpha)n^\alpha}.
	\end{equation}
	Hence proved.
\end{prf1}

\subsection{Proof of Theorem \ref{thm:sr-est-muuniv-hpb}}
\label{appendix:proof-sr-est2-hpb}
\begin{prf1}
	Recall that $n_0$ is chosen such that for all $n\ge n_0$, we have $\frac{c}{n^\alpha} L_1^2 < \mu_1$. 
	Notice that 
	\begin{align}
		\sum\limits_{k=1}^{n}\Gamma_k^2 = \sum\limits_{k=1}^{n_0-1}\Gamma_k^2 + \sum\limits_{k=n+0}^{n}\Gamma_k^2. \label{eq:lsplit}
	\end{align}
	We simplify the first term on the RHS as follows:
	\begin{align*}
		\sum\limits_{k=1}^{n_0-1}\Gamma_k^2 &= \sum\limits_{k=1}^{n_0-1} a_k^2(\prod_{j=k}^{n-1}(1-2\mu_1a_j + a_j^2L_1^2) ) \\
		&= \sum\limits_{k=1}^{n_0-1} a_k^2(\prod_{j=k}^{n_0-1}(1-2\mu_1a_j + a_j^2L_1^2) ) (\prod_{j=n_0}^{n}(1-2\mu_1a_j + a_j^2L_1^2) )\\
		&\leq (1+c^2L_1^2)^{n_0} \sum\limits_{k=1}^{n_0-1} a_k^2 (1+c^2L_1^2)^{-k}\prod_{j=n_0}^{n}(1-a_j(2\mu_1 - a_jL_1^2) )\\
		&\leq (1+c^2L_1^2)^{n_0} \sum\limits_{k=1}^{n_0-1} a_k^2 \exp(-\mu_1\sum\limits_{j=n_0}^{n}a_j) \\
		&\leq (1+c^2L_1^2)^{n_0}\exp\left(-\frac{\mu_1 c(n^{1-\alpha}-n_0^{1-\alpha})}{1-\alpha}\right) \sum\limits_{k=1}^{n_0-1} a_k^2(1+c^2L_1^2)^{-k} \\
		&\leq \frac{(1+c^2L_1^2)^{n_0+1}c^2}{c^2L_1^2}\exp\left(-\frac{\mu_1 c(n^{1-\alpha}-n_0^{1-\alpha})}{1-\alpha}\right).\stepcounter{equation}\tag{\theequation}\label{eq:p63}
	\end{align*}
	We now simplify the second term on the RHS of \eqref{eq:lsplit} as follows:
	\begin{align*}
		\sum\limits_{k=n_0}^{n}\Gamma_k^2 &= \sum\limits_{k=n_0}^{n} a_k^2(\prod_{j=k}^{n-1}(1-2\mu_1a_j + a_j^2L_1^2) ) \leq \sum\limits_{k=n_0}^{n} a_k^2 \prod_{j=k}^{n}(1-a_j(2\mu_1 - a_jL_1^2) )\\
		&\leq  \sum\limits_{k=n_0}^{n} a_k^2 \exp(-\mu_1\sum\limits_{j=k}^{n}a_j) \leq  \sum\limits_{k=n_0}^{n} a_k^2 \exp\left(-\frac{\mu_1 c(n^{1-\alpha}-k^{1-\alpha})}{1-\alpha}\right) \\
		&\leq  \exp\left(-\frac{\mu_1 cn^{1-\alpha}}{1-\alpha}\right)\sum\limits_{k=n_0}^{n} \frac{c^2}{k^{2\alpha}}\exp\left(\frac{\mu_1 ck^{1-\alpha}}{1-\alpha}\right) \\
		&\leq \frac{2(\mu_1 c)^{\frac{\alpha}{1-\alpha}}c^2}{1-\alpha} \frac{1}{n^{\alpha}}.\stepcounter{equation}\tag{\theequation}\label{eq:p64}
	\end{align*}
	Using \eqref{eq:p63} and \eqref{eq:p64} in \eqref{eq:lsplit}, we obtain
	\begin{align*}
		\sum\limits_{k=1}^{n}\Gamma_k^2 = \sum\limits_{k=1}^{n_0-1}\Gamma_k^2+\sum\limits_{k=n_0}^{n}\Gamma_k^2
		&\leq \frac{(1+c^2L_1^2)^{n_0+1}c^2}{c^2L_1^2}\exp\left(-\frac{\mu_1 c(n^{1-\alpha}-n_0^{1-\alpha})}{1-\alpha}\right)  +  \\ & \qquad \frac{2(\mu_1 c)^{\frac{\alpha}{1-\alpha}}c^2}{1-\alpha} \frac{1}{n^{\alpha}}.
	\end{align*}
	Using the above bound in \eqref{eq:hpb1}, we obtain
	\begin{align}
		\P(|z_n| - \E[|z_n|] >\varepsilon)  &\leq \exp\left(- \tilde c n\varepsilon^2\right), \label{eq:hpb2}
	\end{align}
	where $\tilde c = \frac{(1-\alpha)}{2(\mu_1 c)^{\frac{\alpha}{1-\alpha}}c^2}$.
	Finally, using the bound on $\E[|z_n|]$ from Theorem \ref{thm:sr-est-muknown} in \eqref{eq:hpb2}, we obtain
	\begin{align*}
		\P\left(|t_n-\sr(X)| \le C_2\exp\left(-\frac{\mu_1 cn^{1-\alpha}}{2(1-\alpha)}\right) +\frac{C_3}{n^{\alpha/2}}
		\right)\ge 1-\delta, 
	\end{align*}
	where $C_2$ and $C_3$ are as defined in the theorem statement.
	Hence proved.
\end{prf1}

\section{Proofs for SR optimization}
\label{sec:proofs-opt}
\subsection{Proof of Lemma \ref{lem:srprime-consistent}}
\label{appendix:proof-lemma-opt2}
\begin{prf1}
We first bound $\E|B_m(\theta) - B(\theta)|$ as follows:
		\begin{align*}
			&\E[|B_m(\theta) - B(\theta)|] \\
			&= \E\left| \frac{1}{m}\sum_{i=1}^{m}  \ell'\left(\xi_i(\theta) - t_m\right) - \E[\ell'(\xi(\theta) - \sr(\theta))]\right|\\
			&= \E\left| \frac{1}{m}\sum_{i=1}^{m}  \left[\ell'\left(\xi_i(\theta) - t_m\right) - \ell'\left(\xi_i(\theta) - \sr(\theta)\right)\right] + \left[\ell'\left(\xi_i(\theta) - \sr(\theta)\right) - \E[\ell'(\xi(\theta) - \sr(\theta))]\right]\right|\\
			&\le \underbrace{\frac{1}{m}\sum_{i=1}^{m} \E\left| \ell'\left(\xi_i(\theta) - t_m\right) - \ell'\left(\xi_i(\theta) - \sr(\theta)\right)\right|}_{(I)}+ \underbrace{\E\left|\frac{1}{m}\sum_{i=1}^{m}  \ell'\left(\xi_i(\theta) - \sr(\theta)\right) - \E[\ell'(\xi(\theta) - \sr(\theta))]\right|}_{(II)}.\numberthis\label{eq:a11}
		\end{align*}
		Using the fact that $\ell'$ is $L_2$-Lipschitz, we bound the term (I) on the RHS above as follows:
		\begin{align*}
			(I)\le  L_2\E\left| t_m - \sr(\theta)\right| \le \frac{L_2 K_1(m)}{\sqrt{m}},
		\end{align*}
		where the final inequality follows by an application of the bound in Theorem \ref{thm:sr-est-muknown}.
		\begin{align}
			\E[|t_{m}-\sr(X)|] \le \sqrt{\E[(t_{m}-\sr(X))^{2}]} \leq  \frac{K_1(m)}{\sqrt{m}}, \textrm{ where }
		\end{align} 
		$K_1(m)=   \exp\left(\frac{L_1^2c^2\pi^2}{12}\right) \left[ \frac{|t_{0}-\sr(X)|}{m^{\mu_1 c-\frac1{2}}} + \frac{2^{\mu_1 c}c\sigma}{(\mu_1 c - \frac1{2})}\right].$
		
		Next, we bound the term (II)  on the RHS of \eqref{eq:a11} as follows: Letting $\Lambda_i=\ell'\left(\xi_i(\theta) - \sr(\theta)\right)$ and $\Lambda=\ell'\left(\xi(\theta) - \sr(\theta)\right)$,
		\begin{align*}
			(II)= \E\left|\frac{1}{m}\sum_{i=1}^m  \Lambda_i - \E\Lambda\right| &\le \frac{1}{m}\sqrt{\sum_{i=1}^m \E\left(\Lambda_i - \E \Lambda\right)^2} \le \frac{\varl}{\sqrt{m}},\numberthis\label{eq:bmbound}
		\end{align*}
		where the final inequality used the variance bound in \Cref{ass:lossVariance}.
		
		Thus,
		\begin{align*}
			\E[|B_m(\theta) - B(\theta)|] \le \frac{L_2 K_1(m)}{\sqrt{m}} + \frac{\varl}{\sqrt{m}}.
		\end{align*}
		Next, we turn to bounding $\E|A_m(\theta) - A(\theta)|$. 
		Letting $\tilde\Lambda_i=\ell'\left(\xi_i(\theta) - \sr(\theta)\right)\xi'_i(\theta)$ and $\tilde\Lambda=\ell'\left(\xi(\theta) - \sr(\theta)\right)\xi'(\theta)$, we have
		\begin{align*}
			\E[|A_m(\theta) - A(\theta)|] &\le \frac{1}{m}\sum_{i=1}^{m} \E\left| \ell'\left(\xi_i(\theta) - t_m\right)\xi_i'(\theta) - \ell'\left(\xi_i(\theta) - \sr(\theta)\right)\xi_i'(\theta)\right|
			+ \E\left|\frac{1}{m}\sum_{i=1}^m  \tilde\Lambda_i - \E\tilde\Lambda\right|\\
			& \le M_2L_2\E\left| t_m - \sr(\theta)\right| + \frac{1}{m}\sqrt{\sum_{i=1}^m \E\left(\tilde\Lambda_i - \E \tilde\Lambda\right)^2} \\[0.5ex]
			& \le \frac{K_1(m) M_2 L_2}{\sqrt{m}} + \frac{\tilde\varl}{\sqrt{m}}.\numberthis\label{eq:ambound}
		\end{align*}
		Using \eqref{eq:bmbound} and \eqref{eq:ambound}, we arrive at a bound on the UBSR derivative estimation error $\left|h_m'(\theta) - \frac{d\sr(\theta)}{d\theta}\right|$ as follows: 
		\begin{align*}
			\E\left|h_m'(\theta) - \frac{d\sr(\theta)}{d\theta}\right| 
			&= \E\left|\frac{A_m(\theta)}{B_m(\theta)} - \frac{A(\theta)}{B(\theta)}\right|\\[0.75ex]
			&\leq \frac{|B(\theta)|\E[|A_m(\theta) -A(\theta)|] + |A(\theta)| \E[|B_m(\theta) - B(\theta)|]}{\eta^2} \\
			&\leq \frac{\sqrt{\beta_1}(K_1(m)M_2L_2 + \tilde\varl)  + \sqrt{\beta_1} M_2(K_1(m)L_2 + \varl)}{\eta^2 \sqrt{m}},
		\end{align*}
		where the final inequality used \eqref{eq:beta1}, Assumptions \ref{ass:Xbdd},  \ref{ass:lossLipschitz} and \ref{ass:xiprimeBound}. This proves the first claim.
		
		\noindent Before we prove the second claim, we first bound 
		$\E\left[\left(B_m(\theta) - B(\theta)\right)^2\right]$ as follows:
		\begin{align*}
			&\E[\left(B_m(\theta) - B(\theta)\right)^2] \\
			&= \E\left( \frac{1}{m}\sum_{i=1}^{m}  \left[\ell'\left(\xi_i(\theta) - t_m\right) - \ell'\left(\xi_i(\theta) - \sr(\theta)\right)\right] + \left[\ell'\left(\xi_i(\theta) - \sr(\theta)\right) - \E[\ell'(\xi(\theta) - \sr(\theta))]\right]\right)^2\\
			&\le \underbrace{2\E\left(\frac{1}{m}\sum_{i=1}^{m}  \ell'\left(\xi_i(\theta) - t_m\right) - \ell'\left(\xi_i(\theta) - \sr(\theta)\right)\right)^2}_{(III)}\\
			&\quad+ 
			\underbrace{2\E\left(\frac{1}{m}\sum_{i=1}^{m}  \ell'\left(\xi_i(\theta) - \sr(\theta)\right) - \E[\ell'(\xi(\theta) - \sr(\theta))]\right)^2}_{(IV)}.\numberthis\label{eq:a22}
		\end{align*}
		Using the fact that $\ell'$ is $L_2$-Lipschitz, we bound the term (I) on the RHS above as follows:
		\begin{align*}
			(III)\le 2 L_2^2\E\left( t_m - \sr(\theta)\right)^2 \le \frac{2L_2^2 K_1(m)^2}{m},
		\end{align*}
		where the final inequality follows by an application of the bound in Theorem \ref{thm:sr-est-muknown}.
		
		Next, we bound the term (IV)  on the RHS of \eqref{eq:a22} as 
		\begin{align*}
			(IV)= 2\E\left(\frac{1}{m}\sum_{i=1}^m  \Lambda_i - \E\Lambda\right)^2  \le \frac{\varl^2}{m},
		\end{align*}
		where we used the variance bound in \Cref{ass:lossVariance}.
		Thus,
		\begin{align*}
			&\E[\left(B_m(\theta) - B(\theta)\right)^2] \le \frac{2L_2^2 K_1(m)^2}{m} + \frac{2\varl^2}{m}.\numberthis\label{eq:bmsquarebound}
		\end{align*}
		Along similar lines,
		\begin{align*}
			&\E[\left(A_m(\theta) - A(\theta)\right)^2] \le \frac{2L_2^2 M_2^2 K_1(m)^2}{m} + \frac{2\tilde\varl^2}{m}.\numberthis\label{eq:amsquarebound}
		\end{align*}
		\noindent For the second claim in the statement of the lemma, i.e., $
		\E\left(h_m'(\theta) - \frac{d\sr(\theta)}{d\theta}\right)^2 \leq \frac{C_5}{m}$, we have
		\begin{align*}
			\E\left[\left|h_m'(\theta) - \frac{d\sr(\theta)}{d\theta}\right|^2\right] &= \E\left[\left|\frac{A_m(\theta)}{B_m(\theta)}- \frac{A(\theta)}{B(\theta)}\right|^2\right]\\
			&= \E\left[\left|\frac{B(\theta)A_m(\theta) - A(\theta)B(\theta) + A(\theta)B(\theta)- A(\theta)B_m(\theta) }{B_m(\theta)B}\right|^2\right]\\
			&= \E\left[\left|\frac{B(\theta)(A_m(\theta) - A(\theta)) - A(\theta)(B_m(\theta)- B(\theta)) }{B_m(\theta)B}\right|^2\right]\\
			&\leq \frac{2B^2(\theta)\E[\left(A_m(\theta) - A(\theta)\right)^2] +2 A^2(\theta)\E[|B_m(\theta)- B(\theta)|^2] }{\eta^4}\\
			&\leq \frac{2\beta_1(2 L_2^2 K_1(m)^2 + 2\varl^2) + 2\beta_1M_2^2(2 L_2^2 M_2^2 K_1(m)^2 + 2\tilde\varl^2)}{\eta^4m}\\
			&= \frac{C_5}{m},
		\end{align*}
		where the final inequality used \eqref{eq:beta1}, \eqref{eq:bmsquarebound}, \eqref{eq:amsquarebound} and  Assumption \ref{ass:xiprimeBound}.
		Hence proved.
\end{prf1}

\subsection{Proof of Theorem \ref{thm:sr-opt-muknown}}
\label{appendix:proof-sr-opt1}

In order to derive a non asymptotic bound for the last iterate in Theorem \ref{thm:sr-opt-muknown}, we need an upper bound on the second derivative of $\sr(\theta)$. This bound is provided in the following lemma.
	\begin{lemma}
		\label{lem:srdoubleprime}
		Suppose \Crefrange{ass:1}{ass:2} hold for every $\theta \in [\theta_l,\theta_u]$
		and \Crefrange{ass:lprimelowerbound}{ass:xiprimeBound} hold. Then, we have
		\begin{align}
			|h''(\theta)| \leq L_4,
		\end{align}
		where $L_4=\frac{2L_1L_2M_2^2 + 2L_1L_2B_1M_2 + L_1^2L_3}{\eta^2}$, with $L_1,L_2,L_3,M_2$ and $\eta$ as specified in Assumptions \ref{ass:lprimelowerbound}--\ref{ass:xiprimeBound} and $B_1=\frac{L_1M_2}{\eta}$.
	\end{lemma}
	\begin{prf}
		We first provide upper bound for $|h'(\theta)|$. 
		\begin{align}
			|h'(\theta)| &= \frac{|\mathbb{E}[(\ell'(\xi(\theta)-h(\theta))\xi'(\theta)]|}{|\mathbb{E}[(\ell'(\xi(\theta)-h(\theta))]|} 
			\leq \frac{|\mathbb{E}[(\ell'(\xi(\theta)-h(\theta))\xi'(\theta)]|}{\eta} \nonumber\\
			& \leq \frac{|\mathbb{E}[(\ell'(\xi(\theta)-h(\theta))|\xi'(\theta)|]|}{\eta} 
			\leq \frac{L_1M_2}{\eta} = B_1. \label{eq:gradientbound}
		\end{align}
		The expression for $h''(\theta)$ is obtained by differentiating $h'(\theta)$. 
		\begin{align}
			h''(\theta) = \frac{d^2\sr(\theta)}{d\theta^2}
			= \frac{\frac{dA}{d\theta}B - \frac{dB}{d\theta}A}{B^2} ,
		\end{align}
		where $A = \E[\ell'(\xi(\theta)-\sr(\theta))\xi'(\theta)]$, $B = \E[\ell'(\xi(\theta)-\sr(\theta))]$, $\frac{dA}{d\theta} = \E[\ell''(\xi(\theta)-\sr(\theta))(\xi'(\theta)-\frac{d\sr(\theta)}{d\theta})\xi'(\theta) + \ell'(\xi(\theta)-\sr(\theta))\xi''(\theta)]$ and $\frac{dB}{d\theta} = \E[\ell''(\xi(\theta)-\sr(\theta))(\xi'(\theta)-\frac{d\sr(\theta)}{d\theta})]$.
		\begin{align*}
			\frac{dA}{d\theta} &\leq \E\left[|\ell''(\xi(\theta)-\sr(\theta))|\left(|\xi'(\theta)|+\left|\frac{d\sr(\theta)}{d\theta}\right|\right)|\xi'(\theta)| + |\ell'(\xi(\theta)-\sr(\theta))||\xi''(\theta)|\right] \\
			&\leq L_2M_2(M_2+B_1) + L_1L_3
		\end{align*}
		\begin{align*}
			\frac{dB}{d\theta} &\leq \E\left[|\ell''(\xi(\theta)-\sr(\theta))|\left(|\xi'(\theta)|+\left|\frac{d\sr(\theta)}{d\theta}\right|\right) \right] \\
			&\leq L_2(M_2+B_1)
		\end{align*}
		Using assumptions \ref{ass:lossLipschitz} to \ref{ass:xiprimeBound}, we bound the absolute value of $h''(\theta)$ as follows:
		\begin{align*}
			|h''(\theta)| &\leq \frac{\left|\frac{dA}{d\theta}B\right| + \left|\frac{dB}{d\theta}A\right|}{B^2} \\
			&\leq \frac{(L_2M_2^2 +L_2B_1M_2 + L_1L_3)L_1 + (L_2M_2+L_2B_1)L_1M_2}{\eta^2} \\
			&= \frac{2L_1L_2M_2^2 + 2L_1L_2B_1M_2 + L_1^2L_3}{\eta^2}.
		\end{align*}
	\end{prf}

\subsection*{Proof of Theorem \ref{thm:sr-opt-muknown}}
\begin{prf1}
	We first rewrite the update rule \eqref{eq:sr-gd-update} as follows:
	\begin{align*}
		\theta_{n} & = \theta_{n-1} - b_n h'_m(\theta_{n-1})
		= \theta_{n-1} -  b_n\left ( h'(\theta_{n-1}) +  \varepsilon_{n-1}\right),
	\end{align*}
	where $\varepsilon_{n-1} = h'_m(\theta_{n-1})  - h'(\theta_{n-1})$. 
	
	Letting $z_n = \theta_{n} - \theta^*$, we have
	\begin{align*}
		z_{n} & = z_{n-1} -  b_n\left ( h'(\theta_{n-1}) +  \varepsilon_{n-1}\right).
	\end{align*}
	Let $M_{k}  = \int\limits_{0}^{1} [h''(m \theta_{k} + (1-m)\theta^{*})]dm$. Then,
	\begin{align*}
		h'(\theta_n)  &= \int\limits_{0}^{1} [h''(m \theta_{k} + (1-m)\theta^{*})]dm (\theta_n-\theta^*)=M_n z_n, \textrm{ and}\\
		z_{n} & = z_{n-1}(1 -  b_nM_{n-1}) - a_n\varepsilon_{n-1}.
	\end{align*}
	Unrolling the equation above, we obtain
	\begin{align*}
		z_{n} &= z_{0}\prod\limits_{k=1}^{n}(1 - b_kM_{k-1}) - \sum\limits_{k=1}^{n}[b_k \varepsilon_{k-1} \prod\limits_{j=k+1}^{n}(1 - b_jM_{j-1})].
	\end{align*}
	Taking expectations, using Jensen's inequality together with the fact $\left\| a - b\right\|^2 \le 3 \left\| a \right\|^2 + 3\left\| b\right\|^2$, we obtain
	\begin{align*}
		\E&[(z_{n})^2] \\
		&\leq 3 \E[z_0^2]\prod\limits_{k=1}^{n}(1 - b_kM_{k-1})^2 + 3 \E[\sum\limits_{k=1}^{n}[b_k \varepsilon_{k-1} \prod\limits_{j=k+1}^{n}(1 - b_jM_{j-1})]^2\\
		&\leq 3 \E[z_0^2]\prod\limits_{k=1}^{n}(1 - b_kM_{k-1})^2 + 3 \E[(\sum\limits_{k=1}^{n}[b_k \varepsilon_{k-1} \prod\limits_{j=k+1}^{n}(1 - b_jM_{j-1}))^2]\\
		&\leq 3\E[z_0^2](\mathcal{P}_{1:n})^2 + 3 \E[(\sum\limits_{k=1}^{n}b_k \varepsilon_{k-1} \mathcal{P}_{k+1:n})^2]\tag{where $\mathcal{P}_{i:j} =\prod\limits_{k=i}^{j}(1 - b_kM_{k-1})$}\\
		&\leq 3\E[z_0^2] \exp\left(\frac{b^2L_4^2\pi^2}{6}\right)n^{-2\mu_2 b} + 3 \E[(\sum\limits_{l=1}^{n}\sum\limits_{k=1}^{n}[b_kb_l\varepsilon_{l-1} \varepsilon_{k-1} \mathcal{P}_{k+1:n}\mathcal{P}_{l+1:n})]\\
		&\leq 3\E[z_0^2]\exp\left(\frac{b^2L_4^2\pi^2}{6}\right)n^{-2\mu_2 b} + 3 \E[\sum\limits_{k=1}^{n}b_k^2 \varepsilon_{k-1}^2 (\mathcal{P}_{k+1:n})^2 + \sum\limits_{k\neq l}^{n}b_kb_l\varepsilon_{l-1} \varepsilon_{k-1} \mathcal{P}_{k+1:n}\mathcal{P}_{l+1:n}]\\
		&\leq 3\E[z_0^2]\exp\left(\frac{b^2L_4^2\pi^2}{6}\right)n^{-2\mu_2 b} + 3 \underbrace{\sum\limits_{k=1}^{n}\frac{b^2}{k^2} \E[\varepsilon_{k-1}^2] (\mathcal{P}_{k+1:n})^2}_\text{I} + \\ &\qquad 3\underbrace{\sum \limits_{k\neq l}^{n}b_kb_l\sqrt{\E[\varepsilon_{l-1}^2]\E[\varepsilon_{k-1}^2]}\mathcal{P}_{k+1:n}\mathcal{P}_{l+1:n}}_\text{II},\stepcounter{equation}\tag{\theequation}\label{eq:p67}
	\end{align*}
where the term (II) in the final inequality is obtained using Cauchy-Schwartz inequality.

We bound $\mathcal{P}_{i:j}^2$ as follows:
		\begin{align*}
			\mathcal{P}_{i:j}^2&=\prod\limits_{k=i}^{j}(1 - b_kM_{k-1})^2 = \prod\limits_{k=i}^{j}(1 + b_k^2M_{k-1}^2 - 2b_kM_{k-1}) \leq \exp{\sum_{k=i}^{j}(b_k^2M_{k-1}^2 - 2b_kM_{k-1})}\\ &\leq \exp\left(\sum_{k=i}^{j}(b_k^2L_4^2 - 2b_k\mu_2)\right) \leq \exp\left(\frac{b^2L_4^2\pi^2}{6}\right)e^{-\sum_{k=i}^{j}2b_k\mu_2} \leq \exp\left(\frac{b^2L_4^2\pi^2}{6}\right)\left(\frac{i}{j}\right)^{2\mu_2b}.
		\end{align*}
		We now bound term (I) using Lemma \ref{lem:srprime-consistent} as follows:
		\begin{align*}
			I = \sum\limits_{k=1}^{n}\frac{b^2}{k^2} \E[\varepsilon_{k-1}^2] (\mathcal{P}_{k+1:n})^2
			&\leq C_5\exp\left(\frac{b^2L_4^2\pi^2}{6}\right)\sum\limits_{k=1}^{n}\frac{b^2}{k^2} \left(\frac{k+1}{n}\right)^{2\mu_2b}\\ &\leq  \qquad \frac{C_5\exp\left(\frac{b^2L_4^2\pi^2}{6}\right)2^{2\mu_2b}b^2}{(2\mu_2b - 1)} \frac{1}{n}.\stepcounter{equation}\tag{\theequation}\label{eq:p81}
	\end{align*}
Next, using Lemma \ref{lem:srprime-consistent}, we bound the term (II) on the RHS of \eqref{eq:p67} as follows:
		\begin{align*}
			II &= \sum \limits_{k\neq l}^{n}b_kb_l\sqrt{\E[\varepsilon_{l-1}^2]\E[\varepsilon_{k-1}^2]}\mathcal{P}_{k+1:n}\mathcal{P}_{l+1:n}\\
			&\leq \left(\sqrt{\frac{C_5}{m}}\right)^2 \sum\limits_{k\neq l}^{n}b_kb_l\mathcal{P}_{k+1:n}\mathcal{P}_{l+1:n}\\
			&\leq  \frac{2C_5}{m} \sum\limits_{k>l} \frac{b^2}{kl}\prod\limits_{j=k+1}^{n}(1 - b_jM_{j-1})\prod\limits_{j=l+1}^{n}(1 - b_jM_{j-1})\\
			&=\frac{2C_5}{m}\sum\limits_{k>l}\frac{b^2}{kl}\prod\limits_{j=k+1}^{n}(1 - b_jM_{j-1})^2\prod\limits_{j=l+1}^{k}(1 - b_jM_{j-1}) \\
			&\leq \frac{2C_5}{m} \sum\limits_{l=1}^{n}\sum\limits_{k= l+1}^{n} \frac{b^2}{kl}\exp\left(\frac{b^2L_4^2\pi^2}{6}\right)\left(\frac{k+1}{n}\right)^{2\mu_2b}\left(\frac{l+1}{k}\right)^{\mu_2b}\\
			&\leq \frac{2C_5}{m}\exp\left(\frac{b^2L_4^2\pi^2}{6}\right)\sum\limits_{l=1}^{n} \frac{b^2}{l}\left(l+1\right)^{\mu_2b}\sum\limits_{k= l+1}^{n}\frac{1}{k^{\mu_2b+1}}\left(\frac{k+1}{n}\right)^{2\mu_2b}\\
			&\leq \frac{2C_5}{m}\exp\left(\frac{b^2L_4^2\pi^2}{6}\right)\frac{2^{3\mu_2b}}{n^{2\mu_2 b}}\sum\limits_{l=1}^{n} b^2l^{\mu_2b-1}  \sum\limits_{k= l+1}^{n}k^{\mu_2b-1} \\
			&\leq \frac{2C_5}{m}\exp\left(\frac{b^2L_4^2\pi^2}{6}\right)\frac{2^{3\mu_2b}}{n^{2\mu_2 b}}\sum\limits_{l=1}^{n} b^2l^{\mu_2b-1}  \frac{(n+1)^{\mu_2 b} - (l+1)^{\mu_2 b}}{\mu_2b}\\
			&\leq  \frac{2C_5}{m}\exp\left(\frac{b^2L_4^2\pi^2}{6}\right)\frac{2^{4\mu_2 b}}{\mu_2b}\frac{b^2}{n^{\mu_2b}}\sum\limits_{l=1}^{n} l^{\mu_2b-1}  \\
			&\leq \frac{2C_5b^2}{m}\exp\left(\frac{b^2L_4^2\pi^2}{6}\right) \frac{ 2^{4\mu_2 b}(n+1)^{\mu_2b}}{\mu_2^2b^2n^{\mu_2b}} \\[1ex]
			&\leq  \frac{2C_5b^2}{m} \exp\left(\frac{b^2L_4^2\pi^2}{6}\right)\frac{ 2^{5\mu_2 b}}{(\mu_2b)^2}.\stepcounter{equation}\tag{\theequation}\label{eq:p82}
	\end{align*}
	The main claim follows by substituting the bounds obtained in \eqref{eq:p81} and \eqref{eq:p82} in \eqref{eq:p67}.
\end{prf1}

\subsection{Proof of Theorem \ref{th:convexsgd}}
\label{sec:proofs-opt-convex}
\begin{prf1}
We state and prove three useful results in the following lemmas, which aid the proof of Theorem \ref{th:convexsgd}.
		
		\begin{lemma} \label{lem:2}
			Suppose \Crefrange{ass:1}{ass:2} hold for all $\theta \in \Theta$ and \Crefrange{ass:lossLipschitz}{ass:xiprimeBound} hold. Then for all $m \geq 1$, 
			\begin{align}
				\mathbb{E}[h'_m(\theta)^2] \leq \frac{C_5}{m} + \frac{2B_1C_4}{\sqrt{m}} + B_1^2.
			\end{align}
		\end{lemma}
		\begin{prf}
			The proof of this lemma follows directly from Lemma \ref{lem:srprime-consistent}. Let $h'(\theta) = \frac{dSR_{\lambda}(\theta)}{d\theta}$.
			Using the fact that $|x| - |y| \leq |x-y|$ for any $x,y \in \mathbb{R}$ followed by an application of Lemma \ref{lem:srprime-consistent}, we obtain
			\begin{align}
				\mathbb{E}[|h'_m(\theta)|]  &\leq \mathbb{E}[|h'_m(\theta) - h'(\theta)|] +\mathbb{E}[|h'(\theta)|]\leq \frac{C_4}{\sqrt{m}} + |h'(\theta)|. \label{eq:biasterm}
			\end{align}
			Using $(|x| - |y|)^2 \leq (|x-y|)^2$ for any $x,y \in \mathbb{R}$, we obtain
			\begin{align*}
				\mathbb{E}[h'_m(\theta)^2] & \leq \mathbb{E}[(h'_m(\theta) - h'(\theta))^2] + 2\mathbb{E}[|h'_m(\theta)|]|h'(\theta)| - h'(\theta)^2\\
				&\leq \frac{C_5}{m} + 2\left(\frac{C_4}{\sqrt{m}} + |h'(\theta)|\right)|h'(\theta)| - h'(\theta)^2 \\
				&=  \frac{C_5}{m} + 2\frac{C_4}{\sqrt{m}}|h'(\theta)| + h'(\theta)^2\leq \frac{C_5}{m} + \frac{2B_1C_4}{\sqrt{m}} + B_1^2,
			\end{align*}
			where the second inequality follows from Lemma \ref{lem:srprime-consistent} and \eqref{eq:biasterm}. The last inequality follows from \eqref{eq:gradientbound}.
		\end{prf}
		\begin{lemma}\label{lem:3}
			Suppose \Crefrange{ass:1}{ass:2} hold for all $\theta \in \Theta$ and \Crefrange{ass:lossLipschitz}{ass:xiprimeBound}, \ref{ass:compactset} and \ref{ass:convexity} hold. Suppose that the update in \eqref{eq:sgdconvex} is performed for $n$ steps with step-size sequence $\{b_k\}_{k=1}^n$. Then for any $1 < k_0 < k_1 \leq n$,
			\begin{align} \label{eq:lem9}
				\sum_{k=k_0}^{k_1}2b_k\mathbb{E}[h(\theta_k) - h(\theta_{k_0})] \leq \sum_{k=k_0}^{k_1}(2b_kD\mathcal{A}_k + b_k^2\mathcal{B}_k),
			\end{align}
			where $\mathcal{A}_k = \frac{C_4}{\sqrt{m_k}}$, $\mathcal{B}_k = \frac{C_5}{m_k} + 2B_1\mathcal{A}_k + B_1^2$.
		\end{lemma}
		\begin{prf}
			Let $\delta_k = h'_m(\theta_k) - h'(\theta_k)$ and $\zeta_k = |\theta_k - \theta_{k_0}|$. From \eqref{eq:sgdconvex}, we obtain
			\begin{align}
				\zeta_{k+1}^2 &= (\Pi_{\Theta}(\theta_k - b_k h_m'(\theta_k)) - \theta_{k_0})^2 \nonumber\\
				& \leq (\theta_k - b_k h_m'(\theta_k) - \theta_{k_0})^2 \label{eq:projectineq}\\
				& = \zeta_k^2 - 2b_k h_m'(\theta_k)(\theta_k - \theta_{k_0}) + b_k^2h_m'(\theta_k)^2 \nonumber\\
				& = \zeta_k^2 -2b_k(\delta_k + h'(\theta_k))(\theta_k - \theta_{k_0}) + b_k^2h_m'(\theta_k)^2 \nonumber\\
				& = \zeta_k^2 - 2b_k\delta_k(\theta_k - \theta_{k_0})- 2b_kh'(\theta_k)(\theta_k - \theta_{k_0}) + b_k^2h_m'(\theta_k)^2. \label{eq:sq1}
			\end{align}
			The inequality in \eqref{eq:projectineq} holds because $\theta_{k_0}$ belongs to the set $\Theta$, and the operator $\Pi_{\Theta}$ is non-expansive.
			
			Taking expectation on both sides of \eqref{eq:sq1}, and using Lemma \ref{lem:2}, we obtain
			\begin{align*}
				\mathbb{E}[\zeta_{k+1}^2] &\leq \mathbb{E}[\zeta_k^2] - 2b_k\mathbb{E}[h'(\theta_k)(\theta_k - \theta_{k_0})] - 2b_k\mathbb{E}[\delta_k(\theta_k - \theta_{k_0})] + b_k^2\left[\frac{C_5}{m_k} + \frac{2B_1C_4}{\sqrt{m_k}} + B_1^2\right]\\
				&\leq \mathbb{E}[\zeta_k^2] - 2b_k\mathbb{E}[h'(\theta_k)(\theta_k - \theta_{k_0})] + 2b_k\frac{C_4}{\sqrt{m_k}}|\theta_k - \theta_{k_0}| + b_k^2\left[\frac{C_5}{m_k} + \frac{2B_1C_4}{\sqrt{m_k}} + B_1^2\right]\\
				& = \mathbb{E}[\zeta_k^2] - 2b_k\mathbb{E}[h'(\theta_k)(\theta_k - \theta_{k_0})] + 2b_k\mathcal{A}_k\zeta_k+ b_k^2\left[\frac{C_5}{m_k} + 2B_1\mathcal{A}_k + B_1^2\right]\\
				& \leq \mathbb{E}[\zeta_k^2] - 2b_k\mathbb{E}[h(\theta_k) - h(\theta_{k_0})]+ 2b_k\mathcal{A}_k\zeta_k+ b_k^2\mathcal{B}_k,\stepcounter{equation}\tag{\theequation}\label{eq:p79}
			\end{align*}
			where the second inequality follows from Lemma \ref{lem:srprime-consistent}, while the last inequality follows from \Cref{ass:convexity}.
			Rearranging the terms in \eqref{eq:p79}, we obtain
			\begin{align*}
				2b_k\mathbb{E}[h(\theta_k) - h(\theta_{k_0})] \leq \mathbb{E}[\zeta_k^2] - \mathbb{E}[\zeta_{k+1}^2] + 2b_k\mathcal{A}_k\zeta_k+ b_k^2\mathcal{B}_k.
			\end{align*}
			Summing over $k=k_0$ to $k_1$ and using \Cref{ass:compactset} to bound $\zeta_k$ with $D$, we get \eqref{eq:lem9}.
		\end{prf}
		\begin{lemma} \label{lem:4}
			Suppose \Crefrange{ass:1}{ass:2} hold for all $\theta \in \Theta$ and \Crefrange{ass:lossLipschitz}{ass:xiprimeBound}, \ref{ass:convexity} and \ref{ass:compactset} hold. Then, with $b_k = b$ and $m_k = m$, $\forall k\geq 1$,
			\begin{align} \label{eq:lem10}
				\sum_{k=1}^n \mathbb{E}[h(\theta_k) - h(\theta^*)] \leq \frac{D^2}{2b} + 2nD\mathcal{A} + \frac{nbB_1^2}{2},
			\end{align}
			where $\mathcal{A} = \frac{C_4}{\sqrt{m}}$.
		\end{lemma}
		\begin{prf}
			Let $\delta_k = h'_m(\theta_k) - h'(\theta_k)$ and $\rho_{k+1} = \theta_k - b_k(h'(\theta_k)+\delta_k)$. Using convexity of $h(\theta)$, we obtain
			\begin{align}
				h(\theta_k) - h(\theta^*) &\leq h'(\theta_k)(\theta_k - \theta^*) 
				= \left(\frac{\theta_k - \rho_{k+1}}{b_k} - \delta_k\right)(\theta_k - \theta^*) \nonumber \\
				& = \frac{1}{b_k}(\theta_k - \rho_{k+1} - b_k\delta_k)(\theta_k - \theta^*) \nonumber \\
				& = \frac{1}{2b_k}\left((\theta_k - \theta^*)^2 + (\theta_k - \rho_{k+1} - b_k\delta_k)^2 - (\rho_{k+1} - \theta^* + b_k\delta_k)^2\right) \label{eq:lem10_2}\\
				& = \frac{1}{2b_k}\left((\theta_k - \theta^*)^2 - (\rho_{k+1}-\theta^* + b_k\delta_k)^2\right) + \frac{b_k}{2}h'(\theta_k)^2, \nonumber
			\end{align}
			where the equality in \eqref{eq:lem10_2} is obtained using $x  y = \frac{1}{2}(x^2+y^2-(x-y)^2)$. 
			
			Using $h'(\theta_k) \leq B_1$ from \eqref{eq:gradientbound}, we obtain
			\begin{align*}
				h(\theta_k) - h(\theta^*) & \leq \frac{1}{2b_k}\left((\theta_k - \theta^*)^2 - (\rho_{k+1}-\theta^*)^2 - b_k^2\delta_k^2 - 2b_k(\rho_{k+1}-\theta^*)\delta_k\right) + \frac{b_k}{2}B_1^2 \\
				& \leq \frac{1}{2b_k}\left((\theta_k - \theta^*)^2 - (\rho_{k+1}-\theta^*)^2 - 2b_k(\rho_{k+1}-\theta^*)\delta_k\right) + \frac{b_k}{2}B_1^2.
			\end{align*}
			Taking expectations, and using $(\rho_{k+1} - \theta^*)^2 \geq (\theta_{k+1}-\theta^*)^2$, we obtain
			\begin{align}
				\mathbb{E}[h(\theta_k) - h(\theta^*)] & \leq \frac{1}{2b_k}\left(\mathbb{E}[(\theta_k - \theta^*)^2] - \mathbb{E}[(\theta_{k+1} - \theta^*)^2] - 2b_k\mathbb{E}[|\theta_{k+1}-\theta^*||\delta_k|]\right)+ \frac{b_k}{2}B_1^2   \nonumber \\
				& \leq \frac{1}{2b_k}\left(\mathbb{E}[(\theta_k - \theta^*)^2] - \mathbb{E}[(\theta_{k+1} - \theta^*)^2] + 2b_k\mathcal{A}_k\mathbb{E}[|\theta_{k+1}-\theta^*|]\right) + \frac{b_k}{2}B_1^2.  \label{eq:lem10_3}
			\end{align}
			By summing \eqref{eq:lem10_3} over $k$, and using $b_k = b$ and $m_k = m$ along with the inequality $|\theta_k - \theta^*| \leq D,\ \forall k \geq 1$, we obtain \eqref{eq:lem10}.
		\end{prf}
		\subsection*{Proof of Theorem \ref{th:convexsgd}: }
		For $0 \leq i \leq p+1$, define $\nu_i$ as follows:
		\begin{align}
			\nu_i = \text{arg} \inf_{n_i < k \leq n_{i+1}}\mathbb{E}[h(\theta_k)],\ i \in [p+1],\ \text{and}\  \nu_0 = \text{arg} \inf_{\lceil\frac{n}{4}\rceil < k \leq n_1}\mathbb{E}[h(\theta_k)].
		\end{align}
		The horizon $n$ is split into $p$ phases with each phase having a constant step-size and batch-size. We need to show that the final iterate $\theta_n$ is close to an optimal $\theta^*$. Using $\nu_{p+1} = n$, we obtain
		\begin{align} \label{eq:57}
			\mathbb{E}[h(\theta_n)] = \mathbb{E}[h(\theta_{\nu_0})] + \sum_{i=0}^{p}\mathbb{E}[h(\theta_{\nu_{i+1}}) - h(\theta_{\nu_i})].
		\end{align}
		In order to bound $\mathbb{E}[h(\theta_{\nu_{i+1}}) - h(\theta_{\nu_i})]$, consider the case when $i \geq 1$. Using Lemma \ref{lem:3} with $k_0=\nu_i$ and $k_1 = n_{i+2}$, we obtain
		\begin{align}
			\frac{\sum_{k=\nu_i}^{n_{i+2}}2b_k\mathbb{E}[h(\theta_k)-h(\theta_{\nu_i})]}{n_{i+2}-\nu_i+1} & \leq \frac{\sum_{k=\nu_i}^{n_{i+2}}(2b_kD\mathcal{A}_k + b_k^2\mathcal{B}_k)}{n_{i+2}-\nu_i+1} \nonumber\\[1ex]
			& \leq 2b_{n_i+1}D\mathcal{A}_{n_{i}+1} + b_{n_i+1}^2\mathcal{B}_{n_i+1}, \label{eq:58}
		\end{align}
		where the inequality in \eqref{eq:58} follows from the fact that $b_k$ is a non-increasing sequence and $m_k$ is a non decreasing sequence resulting in $\mathcal{A}_k$ and $\mathcal{B}_k$ being non increasing sequences as well. Also note that $\nu_i \geq n_{i}+1$. Now we define the step-size $b_k$ and the batch size $m_k$ as some polynomial function of $n$ as follows:
		\begin{align}
			b_k = \frac{b_02^{-i}}{n^{\alpha_1}},\ \text{and}\ m_k = 2^i n^{\alpha_2},
		\end{align}
		for some positive constants $b_0$, $\alpha_1$ and $\alpha_2$ when $n_i < k \leq n_{i+1}$, $0 \leq i \leq p$. Substituting $b_k$ and $m_k$ in \eqref{eq:58}, we get
		\begin{align} \label{eq:60}
			\frac{\sum_{k=\nu_i}^{n_{i+2}}2b_k\mathbb{E}[h(\theta_k)-h(\theta_{\nu_i})]}{n_{i+2}-\nu_i+1} & \leq \frac{2DC_4b_02^{-3i/2}}{n^{\alpha_1 + \alpha_2/2}} + \frac{b_0^22^{-2i}}{n^{2\alpha_1}}\left[\frac{C_5}{2^in^{\alpha_2}} + \frac{2B_1C_4}{2^{i/2}n^{\alpha_2/2}} + B_1^2\right].
		\end{align}
		Next, we derive a lower bound for the expression on the left hand side of \eqref{eq:60}. Using $\mathbb{E}[h(\theta_k) - h(\theta_{\nu_i})] \geq 0$ whenever $n_i < k \leq n_{i+1}$, we obtain
		\begin{align}
			\frac{\sum_{k=\nu_i}^{n_{i+2}}2b_k\mathbb{E}[h(\theta_k)-h(\theta_{\nu_i})]}{n_{i+2}-\nu_i+1} & \geq \frac{\sum_{k=n_{i+1}+1}^{n_{i+2}}2b_k\mathbb{E}[h(\theta_k)-h(\theta_{\nu_i})]}{n_{i+2}-\nu_i+1} \nonumber\\
			& \geq 2b_{n_{i+2}} \frac{n_{i+2}-n_{i+1}}{n_{i+2}-n_i}\mathbb{E}[h(\theta_{\nu_{i+1}})-h(\theta_{\nu_i})] \nonumber\\
			& \geq \frac{2b_{n_{i+2}}}{5}\mathbb{E}[h(\theta_{\nu_{i+1}})-h(\theta_{\nu_i})] \nonumber\\
			& = \frac{2^{-i}b_0}{5n^{\alpha_1}}\mathbb{E}[h(\theta_{\nu_{i+1}})-h(\theta_{\nu_i})], \label{eq:61}
		\end{align}
		where the second inequality follows from the assumption $\mathbb{E}[h(\theta_{\nu_{i+1}}) - h(\theta_{\nu_i})] \geq 0$, and the fact that $n_{i+2} - n_{i+1} \geq n_{i+2}-\nu_i+1$. The last inequality follows from Lemma 4 of \cite{bhavsar2021nonasymptotic}. Combining the inequalities in \eqref{eq:60} and \eqref{eq:61}, we obtain
		\begin{align} \label{eq:62}
			\mathbb{E}[h(\theta_{\nu_{i+1}})-h(\theta_{\nu_i})] \leq \frac{10DC_42^{-i/2}}{n^{\alpha_2/2}} + \frac{5b_02^{-i}}{n^{\alpha_1}}\left[\frac{C_5}{2^1n^{\alpha_2}} + \frac{2B_1C_4}{2^{i/2}n^{\alpha_2/2}} + B_1^2\right].
		\end{align}
		The proof for the case when $i=0$ is similar to the above. Using \eqref{eq:62} in \eqref{eq:57}, we obtain
		\begin{align}
			\mathbb{E}[h(\theta_n)] &\leq \mathbb{E}[h(\theta_{\nu_0})] + \sum_{i=0}^p \left(\frac{10DC_42^{-i/2}}{n^{\alpha_2/2}} + \frac{5b_02^{-i}}{n^{\alpha_1}}\left[\frac{C_5}{2^in^{\alpha_2}} + \frac{2B_1C_4}{2^{i/2}n^{\alpha_2/2}} + B_1^2\right]\right) \nonumber \\
			&\leq \mathbb{E}[h(\theta_{\nu_0})] + \sum_{i=0}^{\infty} \left(\frac{10DC_42^{-i/2}}{n^{\alpha_2/2}} + \frac{5b_02^{-i}}{n^{\alpha_1}}\left[\frac{C_5}{2^in^{\alpha_2}} + \frac{2B_1C_4}{2^{i/2}n^{\alpha_2/2}} + B_1^2\right]\right) \nonumber\\
			&= \mathbb{E}[h(\theta_{\nu_0})] + \frac{10DC_4}{n^{\alpha_2/2}(1 - 1/\sqrt{2})} + \frac{5C_5b_0}{3n^{\alpha_1+\alpha_2}(1-1/4)} + \frac{10B_1C_4b_0}{n^{\alpha_1+\alpha_2/2}(1-2^{-3/2})} + \frac{5B_1^2b_0}{n^{\alpha_1}(1-1/2)}\nonumber\\
			&\leq \inf_{\lceil\frac{n}{4}\rceil\leq k\leq n_1}\mathbb{E}[h(\theta_k)] + \frac{35DC_4}{n^{\alpha_2/2}} + \frac{20C_5b_0}{3n^{\alpha_1+\alpha_2}} + \frac{16B_1C_4b_0}{n^{\alpha_1+\alpha_2/2}} + \frac{10B_1^2b_0}{n^{\alpha_1}}. \label{eq:63}
		\end{align}
		For $k \leq n_1$, $b_k = \frac{b_0}{n^{\alpha_1}}$ and $m_k = n^{\alpha_2}$. Using the fact that infimum is smaller than the weighted average, we obtain
		\begin{align}
			\inf_{\lceil\frac{n}{4}\rceil\leq k\leq n_1}\mathbb{E}[h(\theta_k)-h(\theta^*)] & \leq \frac{1}{n_1 - \lceil\frac{n}{4}\rceil+1}\sum_{k=\lceil\frac{n}{4}\rceil}^{n_1}\mathbb{E}[h(\theta_k)-h(\theta^*)] \nonumber \\
			& \leq \frac{2}{n_1}\sum_{k=1}^{n_1}\mathbb{E}[h(\theta_k)-h(\theta^*)] \label{eq:64} \\
			&\leq \frac{2}{n_1}\left[\frac{D^2n^{\alpha_1}}{2b_0} + \frac{2n_1DC_4}{n^{\alpha_2/2}} + \frac{n_1b_0B_1^2}{2}\right] \label{eq:65}\\
			& \leq \frac{4D^2}{b_0n^{1-\alpha_1}} + \frac{4DC_4}{n^{\alpha_2/2}} + \frac{b_0B_1^2}{n^{\alpha_1}}, \label{eq:66}
		\end{align}
		where \eqref{eq:64} follows from $n_1 \leq 2(n_1 - \lceil\frac{n}{4}\rceil + 1)$, \eqref{eq:65} follows from Lemma \ref{lem:4} and \eqref{eq:66} follows from the fact that $n_1 \geq \frac{n}{4}$. Using \eqref{eq:66} in \eqref{eq:63}, we obtain the following:
		\begin{align}
			\mathbb{E}[h(\theta_n) - h(\theta^*)] \leq  \frac{4D^2}{b_0n^{1-\alpha_1}} + \frac{39DC_4}{n^{\alpha_2/2}} + \frac{20C_5b_0}{3n^{\alpha_1+\alpha_2}} + \frac{16B_1C_4b_0}{n^{\alpha_1+\alpha_2/2}} +\frac{(11B_1^2)b_0}{n^{\alpha_1}}. \label{eq:67}
		\end{align}
		The values for $\alpha_1$ and $\alpha_2$ which will result in the tightest bound are $1/2$ and $1$ respectively. Substituting these values, we get the main claim of Theorem \ref{th:convexsgd}.
\end{prf1}

\section{Concluding Remarks and Future Work}
\label{sec:conclusions}
We considered the problem of estimating Utility Based Shortfall Risk (UBSR) in an online setting, when samples from the underlying loss distribution are available one sample at a time. We cast the UBSR estimation problem as a stochastic approximation based root finding scheme. We derived non-asymptotic convergence guarantees on the mean-squared error of our UBSR estimator for different step sizes. We also derived high probability bounds for the concentration of the estimation error. 

Finally we considered the UBSR optimization problem, when the loss distribution belongs to a parameterized family. We proposed a stochastic gradient descent scheme, and derived non-asymptotic convergence guarantees under finite second moments. We faced the challenge of working with biased gradient estimates, which we addressed using batching. More broadly, the techniques developed in this work are applicable  in a variety of settings, to characterize the finite sample performance of stochastic approximation and SGD algorithms.

We list a few interesting directions of future research.
First, it would be interesting to explore UBSR optimization in a risk-sensitive reinforcement learning setting. 
Second, our contributions in the context of UBSR optimization can be extended to a vector parameter setting. For this purpose, one could either extend the UBSR derivative expression to cover a vector parameter and subsequently, devise a sample average approximation to UBSR gradient. Alternatively, one could use a gradient estimation scheme based on finite differences, and the simultaneous perturbation method.
Third, it remains open to extend our finite sample bounds to the estimator of multivariate shortfall risk measure (MSRM) \cite{armenti2018}. In a recent work \cite{kaakai2022UBSR}, the authors derive asymptotic consistency and normality results for an estimator of MSRM, while finite sample bounds are not available. 
Fourth, it would be interesting to study robust variants of UBSR, in the spirit of \cite{bartl2020} for OCE risk and estimation/optimization schemes thereof.
Fifth, it would be interesting to derive analytical results for a generalization of the portfolio management example (solved in the Markowitz-sense) considered in Section \ref{sec:ubsr-opt-portfolio}, i.e., minimize the UBSR value while ensuring a minimum expected return. Variations of this problem have been considered in \cite{gundel2008utility,zhaolin2016convex}, but these works do not fall under the realm of online learning, where the underlying model information is not known explicitly. We think two timescale stochastic approximation \cite[Chapter 6]{borkar2008stochastic} could be a promising approach to solve the aforementioned constrained problem using sample data.
Finally, it would be challenging to consider the UBSR estimation and optimization in a setting with Markovian samples.

\bibliographystyle{MOR_rev/informs2014} 
\bibliography{MOR_rev/ref_new}


\end{document}